%% file: csnsquantile.tex
\documentclass{article}

\usepackage{PRIMEarxiv}
\usepackage{natbib}
\usepackage{algorithm}
\usepackage{algorithmic}

\input{icml2023/preamble}

\newcommand{\megaurl}{\url{https://github.com/microsoft/csrobust}}

\title{Time-uniform confidence bands for the CDF under nonstationarity}

\author{
  Paul Mineiro \\
  Microsoft Research \\
  \texttt{pmineiro@microsoft.com}
  \And
  Steve Howard \\
  LinkedIn \\
  \texttt{steve@stevehoward.org}
}

\graphicspath{{icml2023/}}

\begin{document}

\maketitle

\vspace{-24pt}

\input{icml2023/body}

\section*{Acknowledgements}

The authors thank Ian Waudby-Smith for insightful discussion and review.

\bibliographystyle{plainnat}
\bibliography{icml2023/csnsquantile}

\input{icml2023/appendix}

\end{document}

%% file: icml2023/preamble.tex
\usepackage{microtype}
\usepackage{graphicx}
\usepackage{subfigure}
\usepackage{booktabs} %
\usepackage{multirow}
\usepackage{tablefootnote}

\usepackage{hyperref}

\usepackage{pifont}%
\newcommand{\cmark}{\ding{51}}%
\usepackage{amsmath}
\usepackage{amssymb}
\usepackage{mathtools}
\usepackage{amsthm}
\usepackage{thm-restate}
\usepackage{mathabx}

\DeclareMathOperator{\CDF}{CDF}
\newcommand{\AvCDF}{\overline{\CDF}}

\DeclareMathOperator*{\esssup}{ess\,sup}
\DeclareMathOperator{\uboracle}{\Xi}
\DeclareMathOperator{\lboracle}{\Lambda}
\newcommand{\suffstat}{\Psi}

\newcommand{\PP}{\mathbb{P}}
\newcommand{\RR}{\mathbb{R}}
\newcommand{\NN}{\mathbb{N}}

\newcommand{\infseqt}[1]{(#1)_{t=1}^\infty}

\usepackage{float}
\usepackage{nicefrac}

\usepackage[capitalize,noabbrev]{cleveref}

\theoremstyle{plain}
\newtheorem{theorem}{Theorem}[section]

\newtheorem{lemma}[theorem]{Lemma}

\theoremstyle{definition}
\newtheorem{definition}[theorem]{Definition}

\theoremstyle{remark}

\usepackage[textsize=tiny]{todonotes}

\usepackage{caption}

\newcommand{\repeatcaption}[2]{%
  \renewcommand{\thefigure}{\ref{#1}}%
  \captionsetup{list=no}%
  \caption{#2}%
  \addtocounter{figure}{-1}%
}

\usepackage{transparent}
\newcommand{\IIF}[1]{\STATE\algorithmicif\ #1\ \algorithmicthen}
\newcommand{\ENDIIF}{\unskip\ \algorithmicend\ \algorithmicif}
\newcommand{\algcommentlight}[1]{\textcolor{blue!70!black}{\transparent{0.5}\small{\textbf{//\hspace{2pt}#1}}}}

\newcommand{\pythonmethod}[1]{\texttt{#1}}

%% file: icml2023/body.tex
\begin{abstract}
Estimation of the complete distribution of a random variable is
a useful primitive for both manual and automated decision making.
This problem has received extensive attention in the i.i.d. setting,
but the arbitrary data dependent setting remains largely unaddressed.
Consistent with known impossibility results, we present computationally
felicitous time-uniform and value-uniform bounds on the CDF of the running
averaged conditional distribution of a real-valued random variable which
are always valid and sometimes trivial, along with an instance-dependent
convergence guarantee.  The importance-weighted extension is appropriate for
estimating complete counterfactual distributions of rewards given controlled
experimentation data exhaust, e.g., from an A/B test or a contextual bandit.
\end{abstract}

\section{Introduction}
\label{sec:introduction}

What would have happened if I had acted differently?  Although this
question is as old as time itself, successful companies have recently
embraced this question via counterfactual estimation of outcomes from
the exhaust of their controlled experimentation platforms, e.g., based
upon A/B testing or contextual bandits. These experiments are run in
the real (digital) world, which is rich enough to demand statistical
techniques that are non-asymptotic, non-parametric, and non-stationary.
Although recent advances admit characterizing counterfactual average
outcomes in this general setting, counterfactually estimating a
complete distribution of outcomes is heretofore only possible with
additional assumptions.  Nonethless, the practical importance of
this problem has motivated multiple solutions:
see \cref{tab:comparepriorart} for a summary, and \cref{sec:relatedwork}
for complete discussion.

Intriguingly, this problem is provably impossible in the data dependent
setting without additional assumptions.~\cite{rakhlin2015sequential}
Consequently, our bounds always achieve non-asymptotic coverage, but may
converge to zero width slowly or not at all, depending on the hardness of
the instance. We call this design principle
AVAST (\underline{A}lways \underline{V}alid \underline{A}nd
\underline{S}ometimes \underline{T}rivial).

In pursuit of our ultimate goal, we derive factual distribution estimators
which are useful for estimating the complete distribution
of outcomes from direct experience.

\paragraph{Contributions}

\begin{enumerate}
\item In \cref{subsec:unitinterval} we provide a time and value
uniform upper bound on the CDF of the averaged historical conditional
distribution of a discrete-time real-valued random process.  Consistent
with the lack of sequential uniform convergence of linear threshold
functions~\citep{rakhlin2015sequential}, the bounds are always valid
and sometimes trivial, but with an instance-dependent guarantee: when
the data generating process is smooth qua \citet{block2022smoothed}
with respect to the uniform distribution on the unit interval, the bound
width adapts to the unknown smoothness parameter.
\item In \cref{subsec:extensions} we extend the previous technique to
distributions with support over the entire real line, and further to
distributions with a known countably infinite or unknown nowhere dense set of discrete jumps;
with analogous instance-dependent guarantees.
\item In \cref{subsec:importanceweighted} we extend the previous techniques to importance-weighted random variables, achieving our ultimate goal of
estimating a complete counterfactual distribution of outcomes.
\end{enumerate}

We exhibit our techniques in various simulations in \cref{sec:simulations}.
Computationally our procedures have comparable cost to point estimation
of the empirical CDF, as the empirical CDF is a sufficient statistic.

\defcitealias{chandak2021universal}{UnO21}
\defcitealias{waudby2022anytime}{WS22}
\defcitealias{huang2021off}{HLLA21}
\defcitealias{howard2022sequential}{HR22}
\begin{table*}[t]
\caption{Comparison to prior art for CDF estimation.  See \cref{sec:relatedwork} for details.}
\label{tab:comparepriorart}
\vskip 0.15in
\begin{minipage}{\textwidth}
\begin{center}
\begin{small}
\begin{sc}
\begin{tabular}{lccccccccr}
\toprule
Reference &
\multicolumn{1}{|p{1.5cm}|}{\centering Quantile-\\ Uniform?} &
\multicolumn{1}{|p{1.4cm}|}{\centering Time-\\ Uniform?} &
\multicolumn{1}{|p{1.7cm}|}{\centering Non-\\ stationary?} &
\multicolumn{1}{|p{1.75cm}|}{\centering Non-\\ asymptotic?} &
\multicolumn{1}{|p{1.75cm}|}{\centering Non-\\ parametric?} &
\multicolumn{1}{|p{1.75cm}|}{\centering Counter-\\ factual?} &
\multicolumn{1}{|p{1.2cm}|}{\centering $w_{\max}$-\\ free?\footnote{$w_{\max}$ free techniques are valid with unbounded importance weights.}} \\
\midrule
\citetalias{howard2022sequential} & \cmark & \cmark &  & \cmark & \cmark &  & N/A \\
\citetalias{huang2021off} & \cmark &  &  & \cmark & \cmark & \cmark &  \\
\citetalias[\citetext{IID}]{chandak2021universal} & \cmark &  &  & \cmark & \cmark & \cmark & \cmark \\
\citetalias[\citetext{NS}]{chandak2021universal} & \cmark &  & \cmark &  &  & \cmark & \cmark \\
\citetalias[\citetext{\S4}]{waudby2022anytime} & \cmark & \cmark &  & \cmark & \cmark & \cmark & \cmark \\
this paper & \cmark & \cmark & \cmark & \cmark & \cmark & \cmark & \cmark \\
\bottomrule
\end{tabular}
\end{sc}
\end{small}
\end{center}
\end{minipage}
\vskip -0.1in
\end{table*}

\section{Problem Setting}
\label{sec:setting}

Let
$\left(\Omega, \mathcal{F}, \left\{ \mathcal{F}_t \right\}_{t \in \mathbb{N}}, \mathbb{P}\right)$
be a probability space equipped with a discrete-time filtration, on
which let $X_t$ be an adapted real valued random process.  Let
$\mathbb{E}_t\left[\cdot\right] \doteq \mathbb{E}\left[\cdot | \mathcal{F}_t\right]$.
The quantity of interest is
\begin{equation}
\AvCDF_t(v) \doteq \frac{1}{t} \sum_{s \leq t} \mathbb{E}_{s-1}\left[1_{X_s \leq v}\right],
\label{eqn:defcdf}
\end{equation}
i.e., the CDF of the averaged historical conditional distribution.
We desire simultaneously time and value uniform bounds which hold with
high probability, i.e., adapted sequences of maps $L_t, U_t: \mathbb{R} 
\to [0,1]$ satisfying
\begin{equation}
\mathbb{P}\left( \substack{\forall t \in \mathbb{N} \\ \forall v \in \mathbb{R}}: L_t(v) \leq \overline{\CDF}_t(v) \leq U_t(v) \right) \geq 1 - 2\alpha.
\label{eqn:overallguarantee}
\end{equation}

In the i.i.d. setting, \cref{eqn:defcdf} is independent of $t$
and reduces to the CDF of the (unknown) generating distribution.
In this setting, the classic results of \citet{glivenko1933sulla} and
\citet{cantelli1933sulla} established uniform convergence of linear
threshold functions; subsequently the Dvoretzky-Kiefer-Wolfowitz (DKW) inequality characterized fixed time and value
uniform convergence rates~\cite{dvoretzky1956asymptotic,massart1990tight};
extended later to simultaneously time and value uniform
bounds~\cite{howard2022sequential}. The latter result guarantees an
$O(t^{-1} \log(\log(t)))$ confidence interval width, matching the limit
imposed by the Law of the Iterated Logarithm.

\paragraph{AVAST principle} In contrast, under arbitrary data
dependence, linear threshold functions are not sequentially
uniformly convergent, i.e., the averaged historical empirical CDF
does not necessarily converge uniformly to the CDF of the averaged historical conditional
distribution.~\cite{rakhlin2015sequential}  Consequently, additional
assumptions are required to provide a width guarantee.  In this
paper we design bounds that are \underline{A}lways \underline{V}alid
\underline{A}nd \underline{S}ometimes \underline{T}rivial, i.e., under
worst-case data generation
$\forall t, \exists v: U_t(v) - L_t(v) = O(1)$.  Fortunately our bounds
are also equipped with an instance-dependent width guarantee dependent
upon the smoothness of the distribution to a reference measure qua
\cref{def:smooth}.

\paragraph{Additional Notation}  Let $X_{a:b} = \{ X_s \}_{s=a}^b$
denote a contiguous subsequence of a random process.  Let $\mathbb{P}_t$
denote the average historical conditional distribution, defined as a (random) distribution over the sample space $\mathbb{R}$ by
$\mathbb{P}_t(A) \doteq t^{-1} \sum_{s \leq t} \mathbb{E}_{s-1}\left[1_{X_s \in A}\right]$ for a Borel subset $A$.

\begin{figure*}[t!]
\vskip 0.05in
\begin{minipage}[t]{.49\textwidth}
  \vskip 0pt
  \centering
  \includegraphics[width=1.2\linewidth]{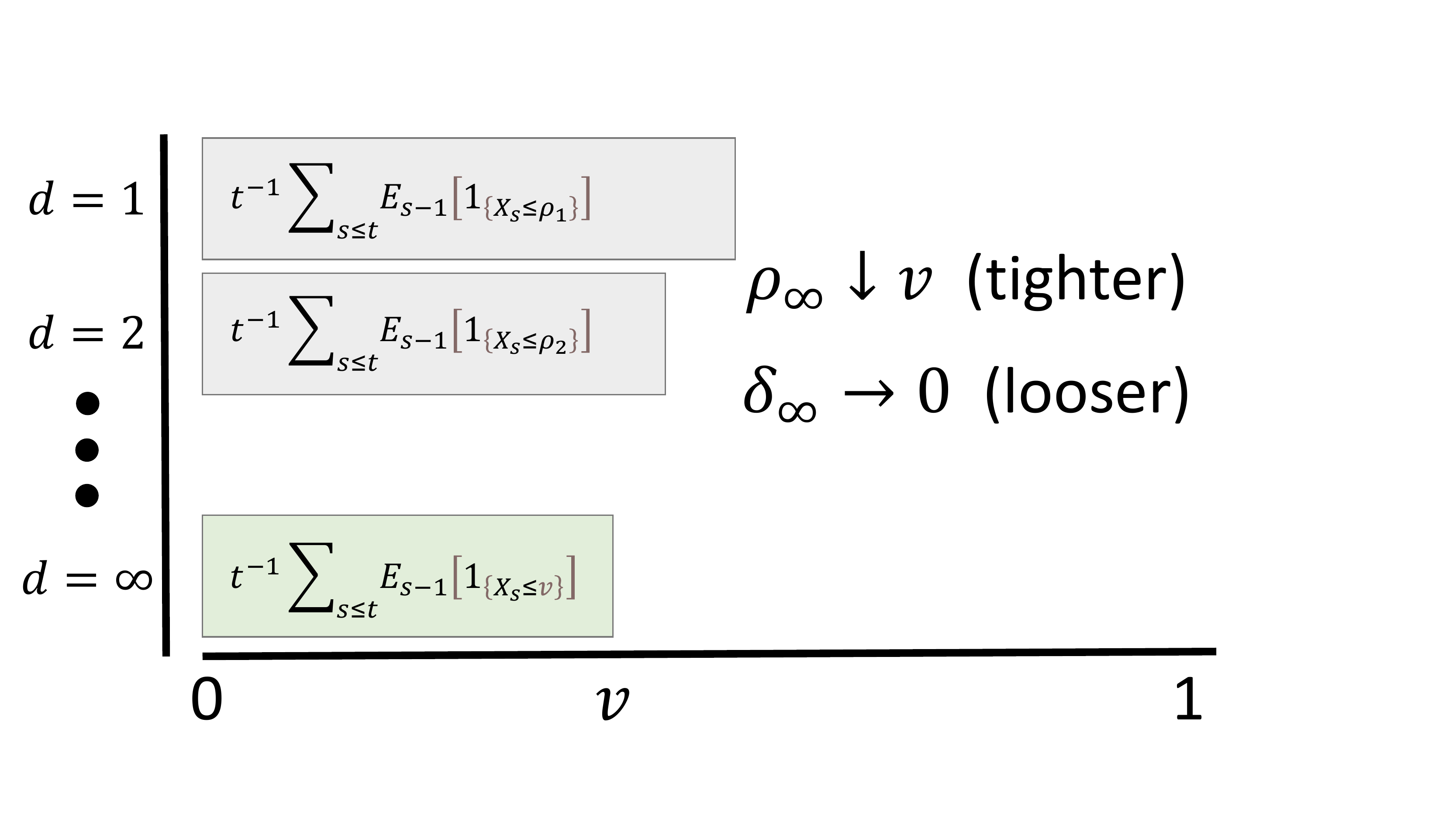}
  \vskip -12pt
  \caption{Visualization of \cref{alg:ubunionzeroone}.  The values of
  interest $v$ are uncountably infinite; the algorithm allocates probability
  $\delta$ to maintain upper bounds on a countably infinite set of points
  $\rho$ at different resolution levels $d$; and leverages the monotonicity
  of $\AvCDF_t(v)$.  The algorithm searches over all $d$ to optimize
  the overall bound via a provably correct early termination criterion.}
  \label{fig:algopic}
\end{minipage}
\hfill
\begin{minipage}[t]{.49\textwidth}
  \begin{algorithm}[H]
     \caption{Unit Interval Upper Bound.  $\epsilon(d)$ is an increasing function specifying the resolution of discretization at level $d$.  $\uboracle_t\left(\rho; \delta, d, \suffstat_t\right)$ is an upper confidence sequence for fixed value $\rho$ with coverage at least $\left(1 - \delta\right)$.}
     \label{alg:ubunionzeroone}
  \begin{algorithmic}
     \STATE {\bfseries Input:} value $v$; confidence $\alpha$; sufficient statistic $\suffstat_t$.
     \STATE \algcommentlight{e.g. $\suffstat_t \doteq X_{1:t}$ or $\suffstat_t \doteq (W_{1:t}, X_{1:t})$}
     \STATE {\bfseries Output:} $U_t(v)$ satisfying \cref{eqn:overallguarantee}.
     \IIF{$v > 1$} \textbf{return} 1 \ENDIIF
     \STATE $u \leftarrow 1$
     \STATE $v \leftarrow \max\left(0, v\right)$
     \FOR{$d=1$ {\bfseries to} $\infty$}
     \STATE $\rho_d \leftarrow \epsilon(d) \lceil \epsilon(d)^{-1} v \rceil$
     \STATE $\delta_d \leftarrow \nicefrac{\alpha}{2^d \epsilon(d)}$
     \STATE $u \leftarrow \min\left(u, \uboracle_t\left(\rho_d; \delta_d, \suffstat_t\right)\right)$
     \IF{$0 = \sum_{s \leq t} 1_{X_s \in \left(v, \rho_d\right]}$}
        \RETURN $u$
     \ENDIF
     \ENDFOR
  \end{algorithmic}
  \end{algorithm}
\end{minipage}
\end{figure*}

\section{Derivations}
\label{sec:derivations}

\subsection{Bounds for the Unit Interval}
\label{subsec:unitinterval}

\paragraph{Fixed $v$} Initially we consider a fixed value $v \in
\mathbb{R}$. Let
$\theta_s \doteq \mathbb{E}_{s-1}\left[1_{X_s \leq v}\right]$
;
$\hat{q}_t \doteq t^{-1} \sum_{s \leq t} 1_{X_s \leq v}$
; and
$q_t \doteq t^{-1} \sum_{s \leq t} \theta_s = \AvCDF_t(v)$.
Our goal is to uniformly bound the deviations of the martingale $S_t = t(\hat{q}_t - q_t)$, quantifying how far the estimand $\AvCDF_t(v)$ lies from the empirical estimate $\hat{q}_t$. We bound these deviations using the following nonnegative martingale,
\newcommand{\eqnbinmart}{E_t(\lambda) \doteq \exp\left(\lambda S_t - \sum_{s \leq t} \log\left(h(\lambda, \theta_s)\right) \right),}
\begin{equation}
\eqnbinmart
\label{eqn:binmart}
\end{equation}
where
$\lambda \in \mathbb{R}$ is fixed
and
$h(\lambda, z) \doteq (1 - z) e^{-\lambda z} + z e^{\lambda (1 - z)}$, the moment-generating function of a centered Bernoulli($z)$ random variable. \cref{eqn:binmart} is a test martingale qua \citet{shafer2011test},
i.e., it can be used to construct time-uniform bounds on $\hat{q}_t - q_t$
via Ville's inequality.  This topic has received extensive
treatment in the literature: we provide a self-contained overview in
\cref{app:confseqreview}.  For now, we will simply posit the existence
of fixed-value confidence sequences $\lboracle_t\left(v; \delta, \suffstat_t\right)$ and
$\uboracle_t\left(v; \delta, \suffstat_t\right)$ based on a sufficient statistic $\suffstat_t$ which, for a fixed
value $v$, achieve
$\mathrm{P}\left(\forall t: \lboracle_t\left(v; \delta, \suffstat_t\right) \leq \AvCDF_t(v) \leq \uboracle_t\left(v; \delta, \suffstat_t\right)\right) \geq 1 - \delta$.
Our goal is to compose these fixed value bounds into a value-uniform
bound, which we achieve via union-bounding over a discretization.

\paragraph{Countably Infinite Union}
In the i.i.d. setting, $q_t$ is independent of $t$, and each $q$ can be
associated with a largest value
$v^{(+)}(q) \doteq \sup\left\{ v | \mathbb{E}\left[1_{X \leq v}\right] \leq q \right\}$;
therefore the upper bound can be evaluated at fixed $q$, i.e.,
$\uboracle_t(q; \delta, \suffstat_t) \doteq \uboracle_t\left(v^{(+)}(q); \delta, \suffstat_t\right)$, to search
over $v$ qua \citet{howard2022sequential}; and analogously for the
lower bound.  This ``quantile space'' approach has advantages, e.g.,
variance based discretization and covariance to monotonic transformations.
Unfortunately, under arbitrary data dependence, $q_t$ changes with time
and \cref{eqn:binmart} does not admit the same strategy, so we proceed
by operating in ``value space''.  See \cref{app:whynotquantilespace}
for more details.

\cref{alg:ubunionzeroone}, visualized in \cref{fig:algopic}, constructs
an upper bound on \cref{eqn:defcdf} which, while valid for all values, is
designed for random variables ranging over the unit interval.  It refines
an upper bound on the value $\rho \geq v$ and exploits monotonicity of
$\AvCDF_t(v)$. A union bound over the (countably infinite) possible
choices for $\rho$ controls the coverage of the overall procedure.
Because the error probability $\delta$ decreases with resolution (and the 
fixed-value confidence radius $\uboracle_t$ increases as $\delta$ decreases), the
procedure can terminate whenever no empirical counts remain between the
desired value $v$ and the current upper bound $\rho$, as all subsequent 
bounds are dominated.

The lower bound is derived analogously in \cref{alg:lbunionzeroone} (which we have left to \cref{app:lowerboundzeroone} for the sake of brevity) and leverages a lower confidence
sequence $\lboracle_t\left(\rho; \delta, \suffstat_t\right)$ (instead of an upper confidence sequence) evaluated at an
increasingly refined lower bound on the value $\rho \leftarrow \epsilon(d)
\lfloor \epsilon(d)^{-1} v \rfloor$.

\begin{theorem}\label{thm:coverage}
If $\epsilon(d) \uparrow \infty$ as $d \uparrow \infty$, then \cref{alg:ubunionzeroone,alg:lbunionzeroone} terminate with probability one. Furthermore, if for all $\rho$, $\delta$, and $d$ the algorithms $\lboracle_t(\rho; \delta, \suffstat_t)$ and $\uboracle_t(\rho; \delta, \suffstat_t)$ satisfy
\begin{align}
P(\forall t: \AvCDF_t(v) \geq \lboracle_t(\rho; \delta, \suffstat_t))
&\geq 1 - \delta, \\
P(\forall t: \AvCDF_t(v) \leq \uboracle_t(\rho; \delta, \suffstat_t))
&\geq 1 - \delta,
\end{align}
then guarantee \eqref{eqn:overallguarantee} holds with $U_t,L_t$ given by the outputs of \cref{alg:ubunionzeroone,alg:lbunionzeroone}, respectively.
\end{theorem}
\begin{proof}
See \cref{app:proofthmcoverage}.
\end{proof}

\Cref{thm:coverage} ensures \cref{alg:ubunionzeroone,alg:lbunionzeroone} yield the desired time- and value-uniform coverage, essentially due to the union bound and the
coverage guarantees of the oracles $\uboracle_t,\lboracle_t$.  However, coverage
is also guaranteed by the trivial bounds $0 \leq \AvCDF_t(v) \leq 1$.
The critical question is: what is the bound width?

\paragraph{Smoothed Regret Guarantee}
Even assuming $X$ is entirely supported on the unit interval, on
what distributions will \cref{alg:ubunionzeroone} provide a non-trivial
bound?  Because each
$\left[\lboracle_t(\rho; \delta, \suffstat_t), \uboracle_t(\rho; \delta, \suffstat_t)\right]$
is a confidence sequence for the mean of the bounded random variable
$1_{X_s \leq \rho}$, we enjoy width guarantees at each of the (countably
infinite) $\rho$ which are covered by the union bound, but the guarantee
degrades as the depth $d$ increases.  If the data generating process
focuses on an increasingly small part of the unit interval over time, the
width guarantees on our discretization will be insufficient to determine
the distribution.  Indeed, explicit constructions demonstrating the
lack of sequential uniform convergence of linear threshold functions
increasingly focus in this manner.~\cite{block2022smoothed}

Conversely, if $\forall t: \AvCDF_t(v)$ was Lipschitz continuous in
$v$, then our increasingly granular discretization would eventually
overwhelm any fixed Lipschitz constant and guarantee uniform convergence.
\cref{thm:smoothedregretzeroone} expresses this intuition, but using the
concept of smoothness rather than Lipschitz, as the former concept will
allow us to generalize further.

\begin{definition}
\label{def:smooth}
A distribution $D$ is $\xi$-smooth wrt reference measure $M$ if $D \ll M$ and $\esssup_M \left(\nicefrac{dD}{dM}\right) \leq \xi^{-1}$.
\end{definition}

\newcommand{\betacurvescaption}{CDF bounds approaching the true CDF when sampling i.i.d.\ from a Beta(6,3) distribution.  Note these bounds are simultaneously valid for all times and values.}
\newcommand{\polyacurvescaption}{Nonstationary Polya simulation for two seeds approaching different average conditional CDFs.  Bounds successfully track the true CDFs in both cases.  See \cref{exp:nonstationary}.}
\begin{figure*}[t]
\centering
\begin{minipage}[t]{.49\textwidth}
  \vskip 0pt
  \centering
  \includegraphics[width=.96\linewidth]{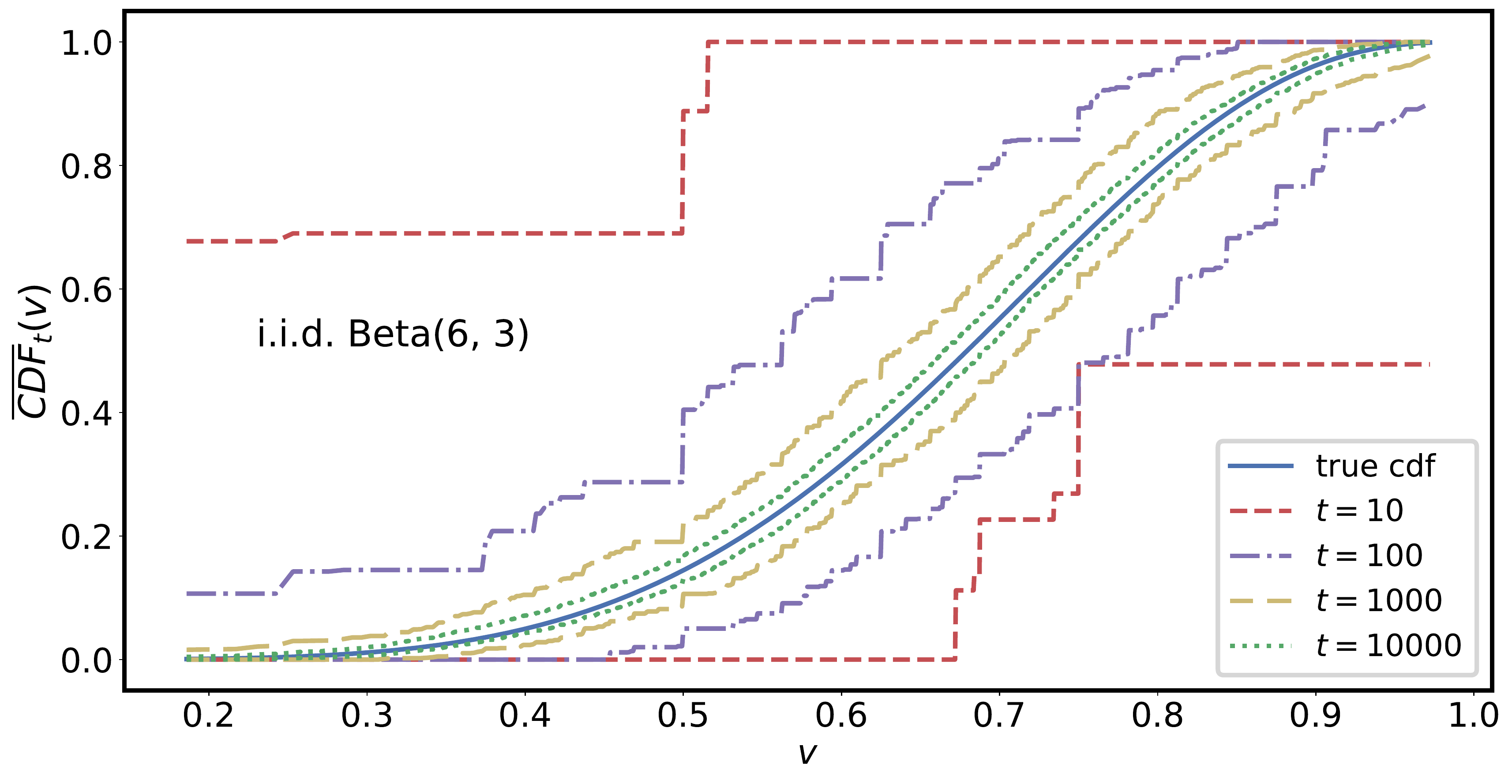}
  \vskip -12pt
  \caption{\betacurvescaption}
  \label{fig:betacurves}
\end{minipage}
\hfill
\begin{minipage}[t]{.49\textwidth}
  \vskip 0pt
  \centering
  \includegraphics[width=.96\linewidth]{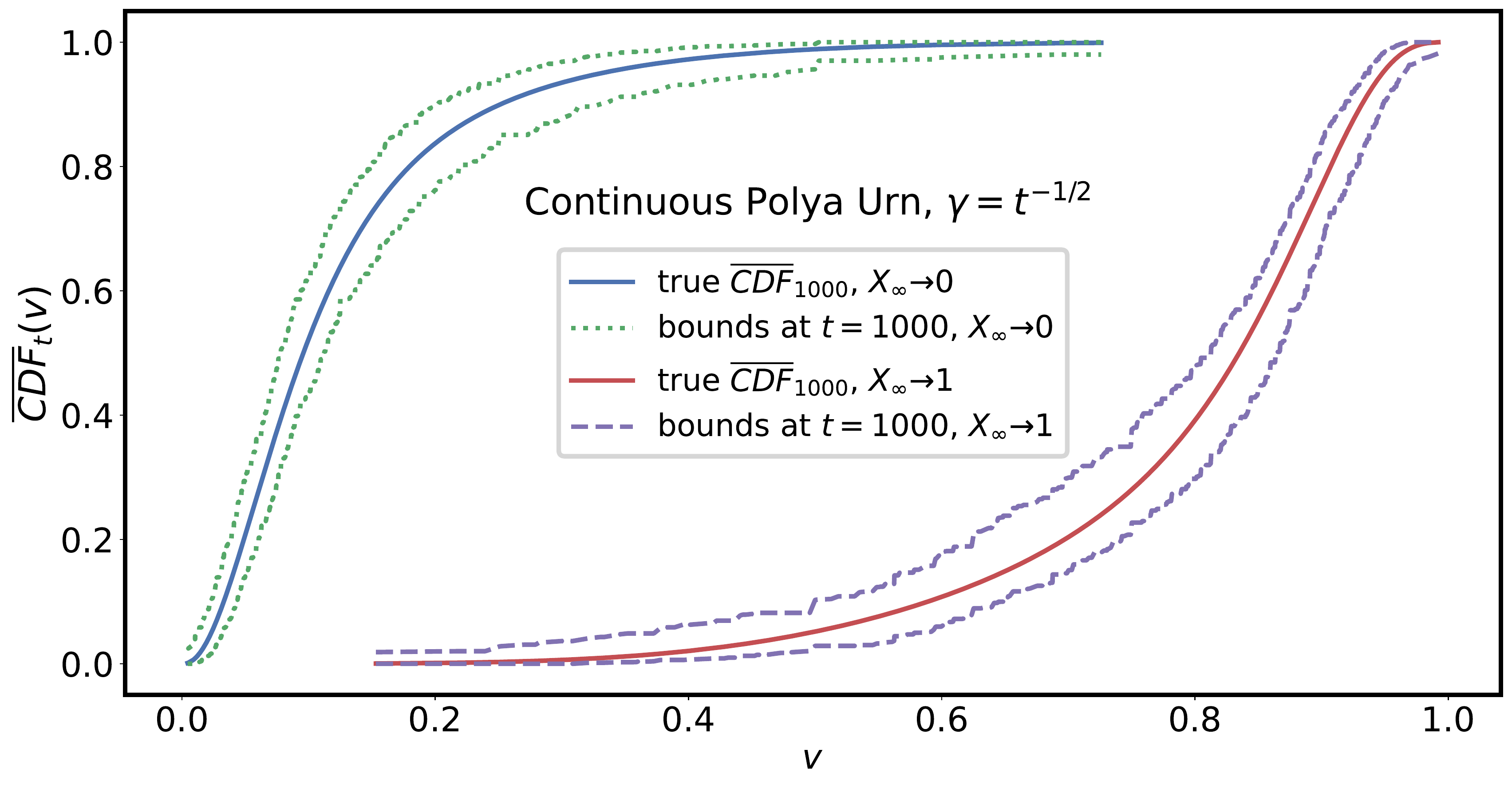}
  \vskip -12pt
  \caption{\polyacurvescaption}
  \label{fig:contpolyatwoseeds}
\end{minipage}
\vskip -0.2in
\end{figure*}

When the reference measure is the uniform distribution on the unit
interval, $\xi$-smoothness implies an $\xi^{-1}$-Lipschitz CDF.  However,
when the reference measure has its own curvature, or charges points,
the concepts diverge.  In this case the reference measure is a probability measure, therefore $\xi \leq 1$.  Note that as $\xi \to 0$, the smoothness constraint
is increasingly relaxed.  As \cref{thm:smoothedregretzeroone} indicates,
when the smoothness constraint relaxes, convergence is slowed.
\begin{restatable}{theorem}{thmSmooth}
\label{thm:smoothedregretzeroone}
Let $U_t(v)$ and $L_t(v)$ be the upper and lower bounds returned
by \cref{alg:ubunionzeroone} and \cref{alg:lbunionzeroone}
respectively, when evaluated with $\epsilon(d) = 2^d$ and the confidence sequences $\lboracle_t$ and $\uboracle_t$
of \cref{eqn:binnormalcurved}.  If $\forall t: \mathbb{P}_t$ is
$\xi_t$-smooth wrt the uniform distribution on the unit interval then
\begin{equation}
\begin{split}
&\forall t, \forall v: U_t(v) - L_t(v) \leq \\
&\qquad \sqrt{\frac{V_t}{t}} + \tilde{O}\left(\sqrt{\frac{V_t}{t} \log\left(\xi_t^{-2} \alpha^{-1} t^{3/2}\right)}\right),
\end{split}
\label{eqn:mainthm}
\end{equation}
where $q_t \doteq \AvCDF_t(v)$;
$V_t \doteq \nicefrac{1}{t} + \nicefrac{(q_t - 1/2)}{\log\left(\nicefrac{q_t}{1-q_t}\right)}$;
and $\tilde{O}()$ elides polylog $V_t$ factors.
\end{restatable}
\begin{proof} See \cref{app:proofsmoothedregretzeroone}.
\end{proof}
\cref{thm:smoothedregretzeroone} matches our empirical results in two
important aspects: (i) logarithmic dependence upon smoothness (e.g.,,
\cref{fig:uniformfail}); (ii) tighter intervals for more extreme quantiles
(e.g., \cref{fig:betacurves}).  Note the choice $\epsilon(d) = 2^d$
ensures the loop in \cref{alg:ubunionzeroone} terminates after at most
$\log_2(\Delta)$ iterations, where $\Delta$ is the minimum difference
between two distinct realized values.

\subsection{Extensions}
\label{subsec:extensions}

\paragraph{Arbitrary Support} In \cref{app:arbitrarysupport} we
describe a variant of \cref{alg:ubunionzeroone} which uses
a countable dense subset of the entire real line.
It enjoys a similar guarantee to \cref{thm:smoothedregretzeroone},
but with an additional width which is logarithmic in the probe
value $v$: 
\newcommand{\logvwidthbound}{ \tilde{O}\left(\sqrt{\frac{V_t}{t} \log\left(\left(2 + \xi_t |v| t^{-1/2}\right)^2 \xi_t^{-2} \alpha^{-1} t^{3/2}\right)}\right)
 }
$\logvwidthbound$.
Note in this case $\xi_t$ is defined relative to (unnormalized) Lebesque
measure and can therefore exceed 1.

\paragraph{Discrete Jumps} If $\mathbb{P}_t$ is smooth wrt a
reference measure which charges a countably infinite number of known
discrete points, we can explicitly union bound over these additional
points proportional to their density in the reference measure.
In this case we preserve the above value-uniform guarantees.  See \cref{app:discretejumps} for more details.

For distributions which charge unknown discrete points, we note the proof of \cref{thm:smoothedregretzeroone} only exploits
smoothness local to $v$.
Therefore if the set of discrete points is nowhere
dense, we eventually recover the guarantee
of \cref{eqn:mainthm} after a ``burn-in'' time
$t$ which is logarithmic in the minimum
distance from $v$ to a charged discrete point.

\newcommand{\iwbetaparetocurvescaption}{CDF bounds approaching the true counterfactual CDF when sampling i.i.d.\ from a Beta(6,3) with infinite-variance importance weights, using DDRM for the oracle confidence sequence.}
\begin{figure*}[t]
\centering
\begin{minipage}[t]{.49\textwidth}
  \vskip 0pt
  \centering
  \includegraphics[width=.96\linewidth]{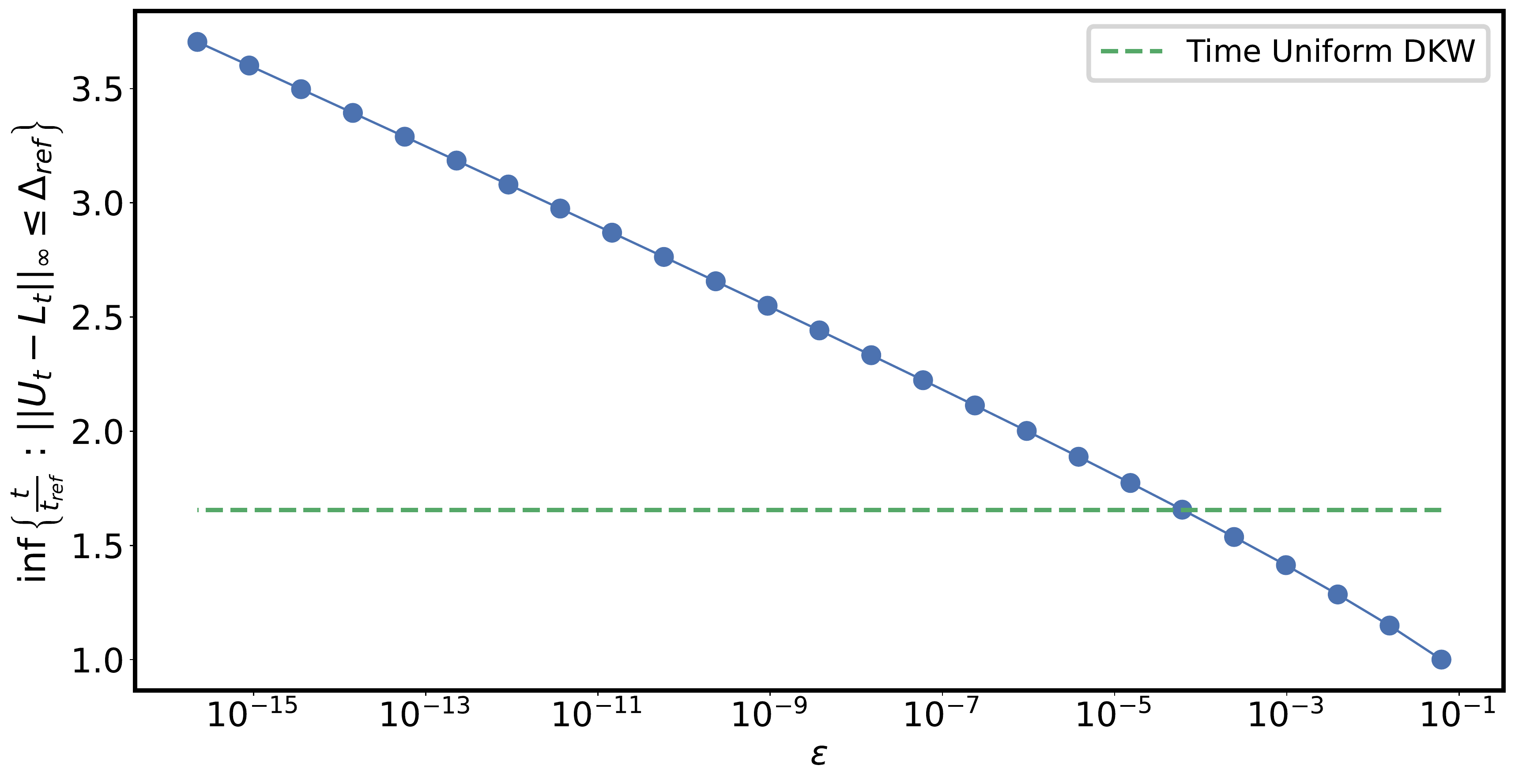}
  \vskip -12pt
  \caption{As smoothness decreases, we require more time to reach the same maximum confidence width. For low smoothness, DKW dominates our method. The logarithmic dependence matches our theory. See \cref{paragraph:fail}.}
  \label{fig:uniformfail}
\end{minipage}
\hfill
\begin{minipage}[t]{.49\textwidth}
  \vskip 1pt
  \centering
  \includegraphics[width=.96\linewidth]{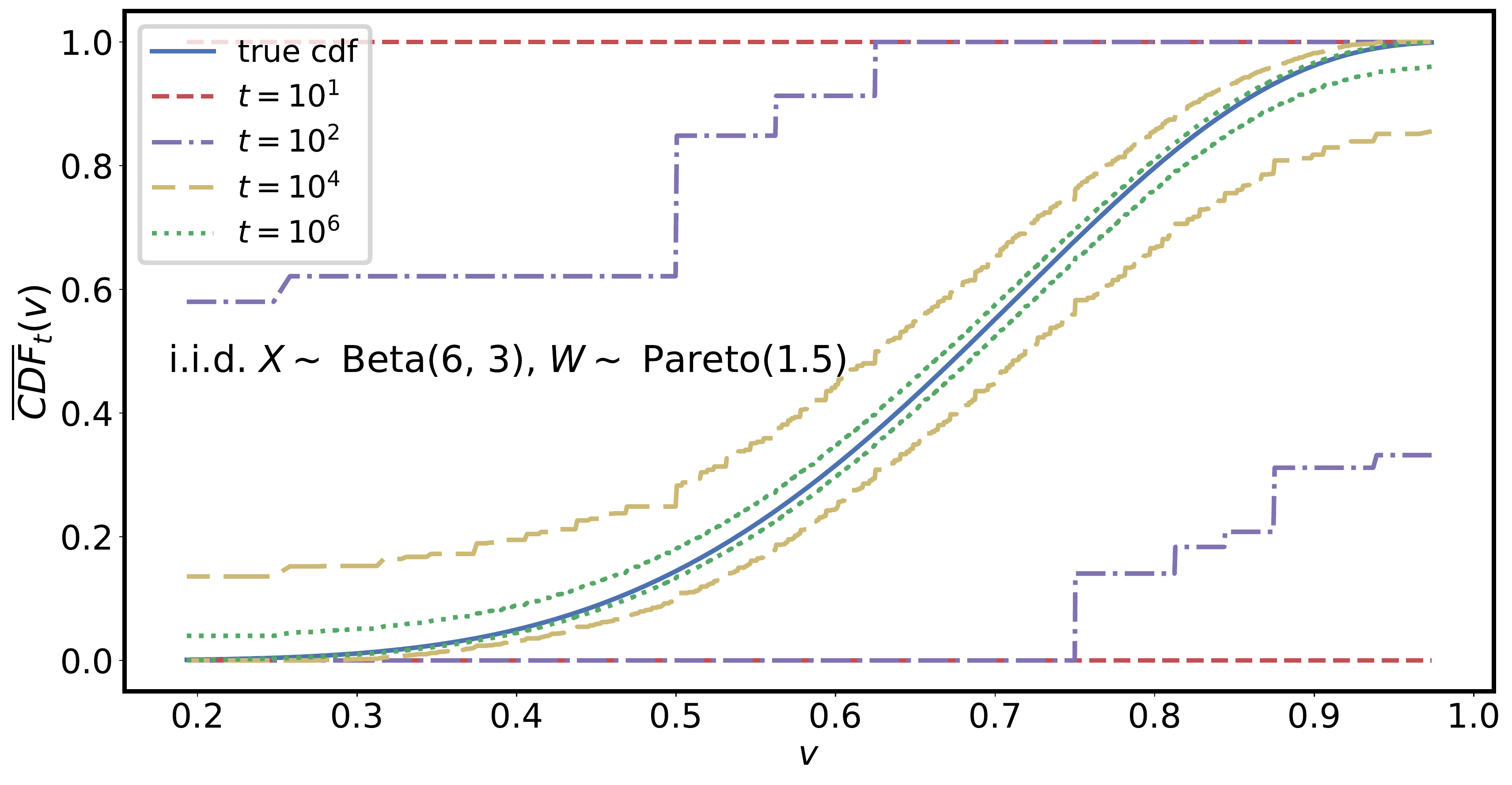}
  \vskip -12pt
  \caption{\iwbetaparetocurvescaption}
  \label{fig:iwexpparetocurves}
\end{minipage}
\vskip -0.2in
\end{figure*}

\subsection{Importance-Weighted Variant}
\label{subsec:importanceweighted}

An important use case is estimating a distribution based upon
observations produced from another distribution with a known shift,
e.g., arising in transfer learning~\cite{pan2010survey} or off-policy
evaluation~\cite{waudby2022anytime}.  In this case the observations are
tuples $(W_t, X_t)$, where the importance weight $W_t$ is a Radon-Nikodym
derivative, implying
$\forall t: \mathbb{E}_t\left[W_t\right] = 1$ and a.s. $W_t \geq 0$;
and the goal is to estimate
$\AvCDF_t(v) = t^{-1} \sum_{s \leq t} \mathbb{E}_{s-1}\left[ W_s 1_{X_s \leq v} \right]$.
The basic approach in \cref{alg:ubunionzeroone} and
\cref{alg:lbunionzeroone} is still applicable in this setting,
but different $\lboracle_t$ and $\uboracle_t$ are required.
In \cref{app:importanceweighted} we present details on two possible
choices for $\lboracle_t$ and $\uboracle_t$: the first is based upon
the empirical Bernstein construction of \citet{howard2021time}, and the
second based upon the DDRM construction of \citet{mineiro2022lower}.
Both constructions leverage the $L^*$ Adagrad bound of
\citet{orabona2019modern} to enable lazy evaluation.  The empirical
Bernstein version is amenable to analysis and computationally lightweight,
but requires finite importance weight variance to converge (the variance
bound need not be known, as the construction adapts to the unknown
variance). The DDRM version requires more computation but produces
tighter intervals.  See \cref{subsec:iidiwexperiments} for a comparison.

Inspired by the empirical Bernstein variant, the following analog of
\cref{thm:smoothedregretzeroone} holds.  Note $\mathbb{P}_t$
is the target (importance-weighted) distribution, not the observation
(non-importance-weighted) distribution.

\begin{restatable}{theorem}{thmIwSmooth}
\label{thm:iwsmoothedregretzeroone}
Let $U_t(v)$ and $L_t(v)$ be the upper and lower bounds returned by
\cref{alg:ubunionzeroone} and \cref{alg:lbunionzeroone} respectively
with $\epsilon(d) = 2^d$ and the confidence sequences $\lboracle_t$
and $\uboracle_t$ of \cref{eqn:iwcurvedboundary}.  If $\forall t:
\mathbb{P}_t$ is $\xi_t$-smooth wrt the uniform distribution on the unit
interval then
\begin{equation}
\begin{split}
&\forall t, \forall v: U_t(v) - L_t(v) \leq \\
&\qquad B_t + \sqrt{\frac{(\tau + V_t)/t}{t}} \\
&\qquad + \tilde{O}\left(\sqrt{\frac{(\tau + V_t)/t}{t} \log\left(\xi_t^{-2} \alpha^{-1}\right)}\right) \\
&\qquad + \tilde{O}(t^{-1} \log\left(\xi_t^{-2} \alpha^{-1}\right)),
\end{split}
\label{eqn:iwmainthm}
\end{equation}
where $q_t \doteq \AvCDF_t(v)$,
$K(q_t) \doteq \nicefrac{(q_t - 1/2)}{\log\left(\nicefrac{q_t}{1-q_t}\right)}$;
$V_t = O\left(K(q_t) \sum_{s \leq t} W_s^2\right)$,
$B_t \doteq t^{-1} \sum_{s \leq t} (W_s - 1)$,
and $\tilde{O}()$ elides polylog
$V_t$ factors.
\end{restatable}
\begin{proof} See \cref{app:proofiwsmoothedregretzeroone}.
\end{proof}
\cref{thm:iwsmoothedregretzeroone} exhibits the following key
properties: (i) logarithmic dependence upon smoothness; (ii)
tighter intervals for extreme quantiles and importance weights
with smaller quadratic variation; (iii) no explicit dependence
upon importance weight range; (iv) asymptotic zero width for
importance weights with sub-linear quadratic variation.

\paragraph{Additional Remarks} First, the importance-weighted average
CDF is a well-defined mathematical quantity, but the interpretation as
a counterfactual distribution of outcomes given different actions in
the controlled experimentation setting involves subtleties: we refer the
interested reader to \citet{waudby2022anytime} for a complete discussion.
Second, the need for nonstationarity techniques for estimating the
importance-weighted CDF is driven by the outcomes $(X_t)$ and not the
importance-weights $(W_t)$.  For example with off-policy contextual
bandits, a changing historical policy does not induce nonstationarity,
but a changing conditional reward distribution does.

\section{Simulations}
\label{sec:simulations}

These simulations explore the empirical
behaviour of \cref{alg:ubunionzeroone} and
\cref{alg:lbunionzeroone} when instantiated
with $\epsilon(d) = 2^d$ and
curved boundary
oracles $\lboracle$ and $\uboracle$. To save
space, precise details on the experiments as
well additional figures are elided to
\cref{app:simulations}.  Reference implementations
which reproduce the figures are available at \megaurl.

\subsection{The i.i.d. setting}

These simulations exhibit our techniques on i.i.d.\ data.  Although
the i.i.d.\ setting does not fully exercise the technique, it is convenient for visualizing
convergence to the unique true CDF.  In this setting
the DKW inequality applies, so to build intuition about our
statistical efficiency, we compare our bounds with a naive time-uniform
version of DKW resulting from a $\left(\nicefrac{6}{\pi^2 t^2}\right)$
union bound over time.

\paragraph{Beta distribution} In this case the data is smooth
wrt the uniform distribution on $[0, 1]$ so we can directly
apply \cref{alg:ubunionzeroone} and \cref{alg:lbunionzeroone}.
\cref{fig:betacurves} shows the bounds converging to the true CDF
as $t$ increases for an i.i.d.  $\text{Beta}(6, 3)$ realization.
\cref{fig:betawidth} compares the bound width to time-uniform DKW at
$t=10000$ for Beta distributions that are increasingly less smooth with
respect to the uniform distribution.  The DKW bound is identical for all,
but our bound width increases as the smoothness decreases.

The additional figures in \cref{app:simulations}
clearly indicate tighter bounds at extreme
quantiles. The analysis of \cref{thm:smoothedregretzeroone} uses the
worst-case variance (at the median) and
does not reflect this.

\newcommand{\lognormalwidthcaption}{Demonstration of the variant described in \cref{subsec:extensions,app:arbitrarysupport} for distributions with arbitrary support, based on i.i.d.\ sampling from a variety of lognormal distributions. Logarithmic range dependence is evident.}

\begin{figure*}[t]
\centering
\begin{minipage}[t]{.49\textwidth}
  \vskip 0pt
  \centering
  \includegraphics[width=.96\linewidth]{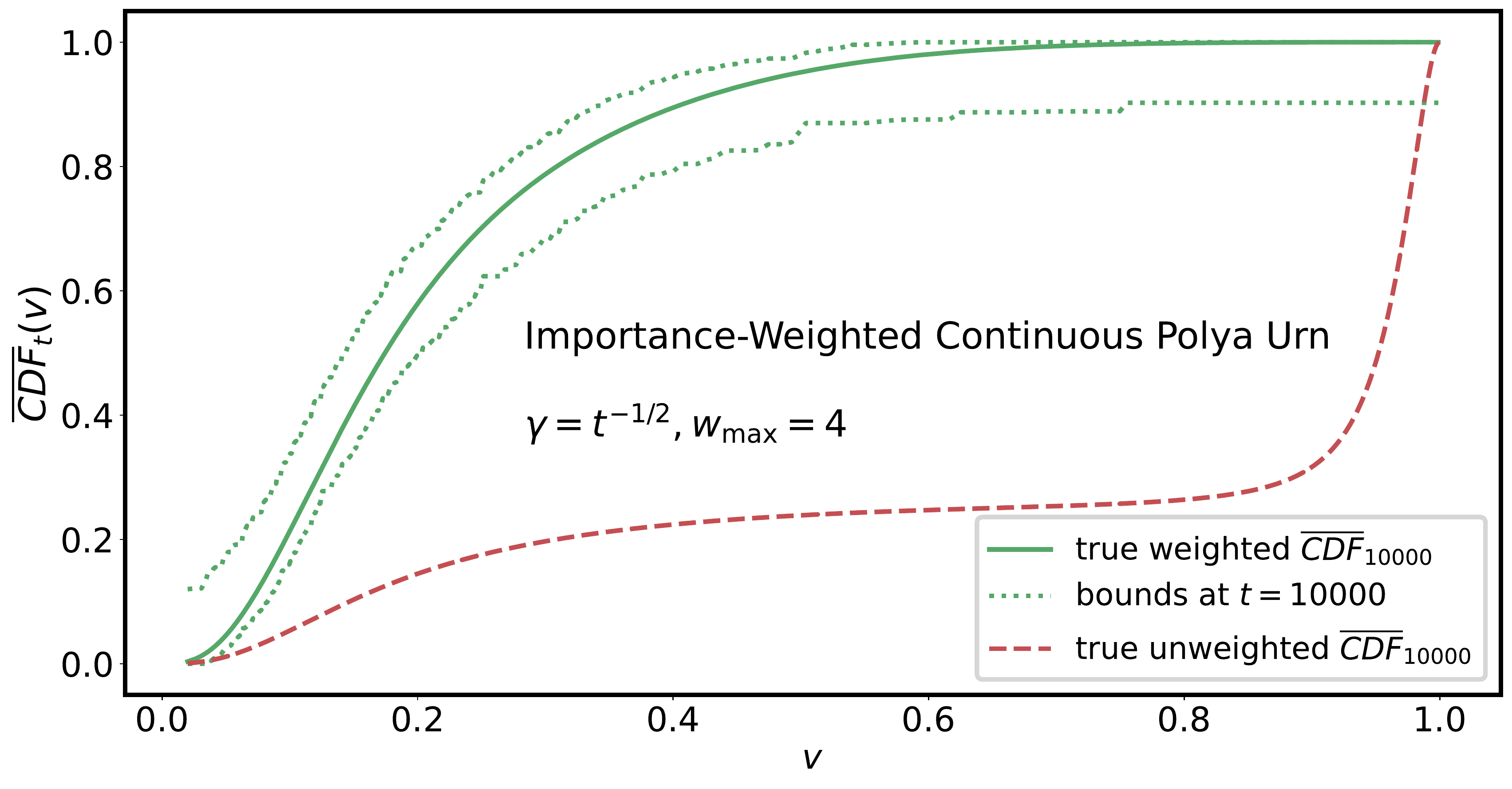}
  \vskip -12pt
  \caption{A nonstationary, importance-weighted simulation in which the factual distribution (red) diverges dramatically from the counterfactual distribution (green). The bound correctly covers the counterfactual CDF.}
  \label{fig:iwcontpolyatwoseeds}
\end{minipage}
\hfill
\begin{minipage}[t]{.49\textwidth}
  \vskip 0pt
  \centering
  \includegraphics[width=.96\linewidth]{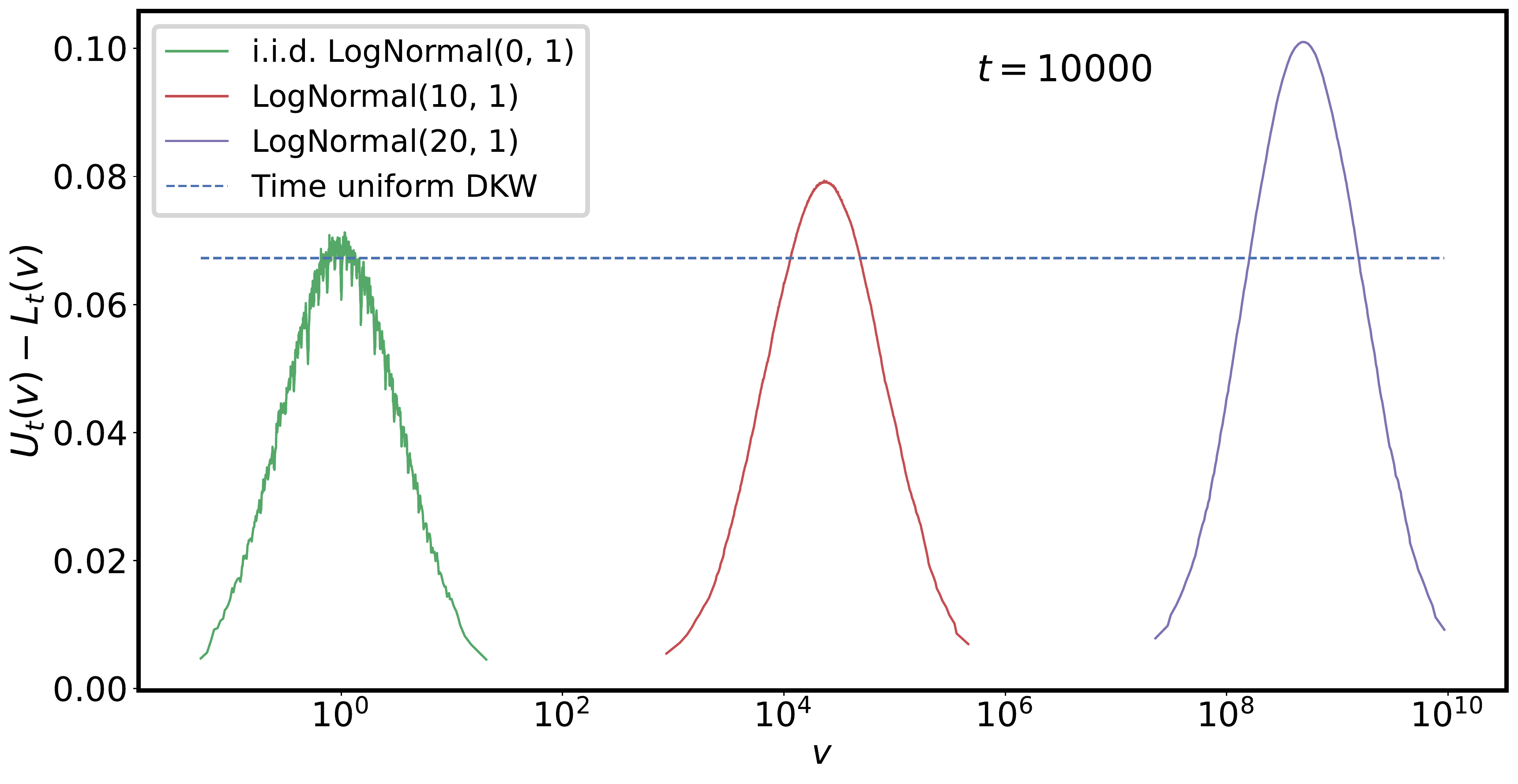}
  \vskip -12pt
  \caption{\lognormalwidthcaption}
  \label{fig:lognormalwidth}
\end{minipage}
\vskip -0.2in
\end{figure*}

\paragraph{Beyond the unit interval} In \cref{fig:lognormalwidth}
(main text) and \cref{app:iidsimulations} we present
further simulations of i.i.d. lognormal and Gaussian random variables,
ranging over $\mathbb{R}^+$ and $\mathbb{R}$ respectively, and using
\cref{alg:ubunionrealline}.  The logarithmic dependence of the
bound width upon the probe value is evident.

\paragraph{An Exhibition of Failure}
\label{paragraph:fail}
\cref{fig:uniformfail}
shows the (empirical) relative convergence
when the data is
simulated i.i.d. uniform over $[0, \epsilon]$ for decreasing $\epsilon$
(hence decreasing smoothness).  The reference
width is the maximum bound width obtained with \cref{alg:ubunionzeroone}
and \cref{alg:lbunionzeroone} at
$t_{\text{ref}} = 10000$ and $\epsilon = 1/16$,
and shown is the multiplicative factor of time required for the maximum bound width
to match the reference width as smoothness varies. The trend is consistent with arbitrarily
poor convergence with arbitrarily small $\epsilon$. Because this is
i.i.d.\ data, DKW applies and a uniform bound (independent of $\epsilon$)
is available.  Thus while our instance-dependent guarantees are valuable
in practice, they can be dominated by stronger guarantees leveraging
additional assumptions. On a positive note, a logarithmic dependence on smoothness is evident
over many orders of magnitude, confirming the analysis of
\cref{thm:smoothedregretzeroone}.

\paragraph{Importance-Weighted}
\label{subsec:iidiwexperiments}
In these simulations, in addition to being i.i.d., $X_t$ and $W_t$
are drawn independently of each other, so the importance weights merely
increase the difficulty of ultimately estimating the same quantity.

In the importance-weighted case, an additional aspect is whether
the importance-weights have finite or infinite variance.
\cref{fig:iwexpbetacurves,fig:iwexpparetocurves}
demonstrate convergence in both conditions when using DDRM for pointwise
bounds.  \cref{fig:iwexpbetacurvesbern,fig:iwexpparetocurvesbern} show the results using empirical
Bernstein pointwise bounds.  In theory, with enough samples and infinite
precision, the infinite variance Pareto simulation would eventually
cause the empirical Bernstein variant to reset to trivial bounds, but
in practice this is not observed.  Instead, DDRM is consistently tighter
but also consistently more expensive to compute, as exemplified in
\cref{tab:compddrmebern}.  Thus either choice is potentially preferable.
\begin{table}[h]
\caption{Comparison of DDRM and Empirical Bernstein on i.i.d. $X_t \sim
\text{Beta}(6, 3)$, for different $W_t$.  Width denotes the maximum
bound width $\sup_v U_t(v) - L_t(v)$.  Time is for computing the bound
at 1000 equally spaced points.}
\label{tab:compddrmebern}
\vskip 0.15in
\begin{center}
\begin{small}
\begin{sc}
\begin{tabular}{lccr}
\toprule
$W_t$ &
What &
Width &
Time (sec) \\
\midrule
\multirow{2}{*}{$\text{Exp}(1)$} & DDRM & 0.09 & 24.8 \\
 & Emp. Bern & 0.10 & 1.0 \\
\multirow{2}{*}{$\text{Pareto}(\nicefrac{3}{2})$} & DDRM & 0.052 & 59.4 \\
 & Emp. Bern & 0.125 & 2.4 \\
\bottomrule
\end{tabular}
\end{sc}
\end{small}
\end{center}
\vskip -0.1in
\end{table}

\subsection{Nonstationary}
\label{exp:nonstationary}

\paragraph{Continuous Polya Urn} In this case
$$
X_t \sim \text{Beta}\left(2 + \gamma_t \sum_{s < t} 1_{X_s > \nicefrac{1}{2}}, 2 + \gamma_t \sum_{s < t} 1_{X_s \leq \nicefrac{1}{2}}\right),$$ i.e., $X_t$ is Beta distributed
with parameters becoming more extreme over time: each realization will
increasingly concentrate either towards $0$ or $1$.  Suppose $\gamma_t
= t^q$.  In the most extreme case that $\left(t = \sum_{s \leq
t} 1_{X_s > \nicefrac{1}{2}}\right)$, the conditional distribution at time $t$
is $\text{Beta}\left(x; 2 + t \gamma_t, 2\right) = O(t^{1+q})$,
hence $\nicefrac{d\mathbb{P}_t}{dU} = O(t^{1 + q})$, which is smooth enough
for our bounds to converge.  \cref{fig:contpolyatwoseeds}
shows the bounds covering the true CDF for two realizations with different
limits.  \cref{fig:contpolyagammasweep} shows (for one realization)
the maximum bound width, scaled by $\sqrt{\nicefrac{t}{\log(t)}}$ to
remove the primary trend, as a function of $t$ for different $\gamma_t$
schedules.

\paragraph{Importance-Weighted Continuous Polya Urn} In this case
$W_t$ is drawn iid either $W_t = 0$ or $W_t = w_{\max}$, such as
might occur during off-policy evaluation with an $\epsilon$-greedy
logging policy. Given $W_t$, the distribution of $X_t$ is given by
$$
\begin{aligned}
X_t | W_t &\sim \text{Beta}\left(2 + \gamma_t \sum_{s < t} 1_{X_s > 1/2} 1_{W_s = W_t}, \right. \\
&\qquad \qquad \left. 2 + \gamma_t \sum_{s < t} 1_{X_s < 1/2} 1_{W_s = W_t}\right),
\end{aligned}
$$
i.e., each importance weight runs an independent Continuous Polya Urn.
Because of this, it is possible for the unweighted CDF to mostly
concentrate at one limit (e.g., 1) but the weighted CDF to concentrate
at another limit (e.g., 0).  \cref{fig:iwcontpolyatwoseeds} exhibits
this phenomenon.

\section{Related Work}
\label{sec:relatedwork}

Constructing nonasymptotic confidence bands for the cumulative distribution function of iid random variables is a classical problem of statistical inference dating back to~\citet{dvoretzky1956asymptotic} and~\citet{massart1990tight}. While these bounds are quantile-uniform, they are ultimately fixed-time bounds (i.e.~not time-uniform). In other words, given a sample of iid random variables $X_1, \dots, X_n \sim F$, these fixed time bounds $[\dot L_n(x), \dot U_n(x)]_{x \in \RR}$ satisfy a guarantee of the form:
\begin{equation}
  \PP(\forall x \in \RR,\ \dot L_n(x) \leq F(x) \leq \dot U_n(x)) \geq 1-\alpha,
\end{equation}
for any desired error level $\alpha \in (0, 1)$. \citet{howard2022sequential} developed confidence bands $[\widebar L_t(x), \widebar U_t(x)]_{x \in \RR, t \in \NN}$ that are both quantile- \emph{and} time-uniform, meaning that they satisfy the stronger guarantee:
\begin{equation}
  \PP(\forall x \in \RR, t \in \NN,\ \widebar L_t(x) \leq F(x) \leq \widebar U_t(x)) \geq 1-\alpha.
\end{equation}

However, the bounds presented in \citet{howard2022sequential} ultimately focused on the classical iid \emph{on-policy} setup, meaning the CDF for which confidence bands are derived is the same CDF as those of the observations~$\infseqt{X_t}$. This is in contrast to off-policy evaluation problems such as in randomized controlled trials, adaptive A/B tests, or contextual bandits, where the goal is to estimate a distribution different from that which was collected (e.g.~collecting data based on a Bernoulli experiment with the goal of estimating the counterfactual distribution under treatment or control). \citet{chandak2021universal} and \citet{huang2021off} both introduced fixed-time (i.e.~non-time-uniform) confidence bands for the off-policy CDF in contextual bandit problems, though their procedures are quite different, rely on different proof techniques, and have different properties from one another. \citet[Section 4]{waudby2022anytime} later developed \emph{time-uniform} confidence bands in the off-policy setting, using a technique akin to \citet[Theorem 5]{howard2022sequential} and has several desirable properties in comparison to \citet{chandak2021universal} and \citet{huang2021off} as outlined in~\citet[Table 2]{waudby2022anytime}.

Nevertheless, regardless of time-uniformity or on/off-policy estimation, all of the aforementioned prior works assume that the distribution to be estimated is \emph{fixed and unchanging over time}. The present paper takes a significant departure from the existing literature by deriving confidence bands that allow the distribution to change over time in a data-dependent manner, all while remaining time-uniform and applicable to off-policy problems in contextual bandits. Moreover, we achieve this by way of a novel stitching technique which is closely related to those of~\citet{howard2022sequential} and~\citet{waudby2022anytime}.

\section{Discussion}
\label{sec:discussion}

This work constructs bounds by tracking specific
values, in contrast with i.i.d.\ techniques which
track specific quantiles.  The value-based approach
is amenable to proving correctness qua
\cref{thm:coverage}, but has the disadvantage
of sensitivity to monotonic transformations.
We speculate it is possible to be covariant
to a fixed (wrt time) but unknown monotonic
transformation without violating known
impossibility results.  A technique with this
property would have increased practical utility.

%% file: icml2023/appendix.tex
\newpage
\appendix
\onecolumn

\section{Confidence Sequences for Fixed $v$}
\label{app:confseqreview}

Since our algorithm operates via reduction to pointwise confidence sequences,
we provide a brief self-contained review here.  We refer the interested
reader to \citet{howard2021time} for a more thorough treatment.

A confidence sequence for a random process $X_t$ is a time-indexed
collection of confidence sets $\text{CI}_t$ with a time-uniform
coverage property $\mathbb{P}\left(\forall t \in \mathbb{N}: X_t \in
\text{CI}_t\right) \geq 1 - \alpha$.  For real random variables,
the concept of a lower confidence sequence can be defined via
$\mathbb{P}\left(\forall t \in \mathbb{N}: X_t \geq L_t\right) \geq 1 -
\alpha$, and analogously for upper confidence sequences; and a lower
and upper confidence sequence can be combined to form a confidence
sequence $\text{CI}_t \doteq \left\{ x | L_t \leq x \leq U_t \right\}$
with coverage $(1 - 2 \alpha)$ via a union bound.

One method for constructing a lower confidence sequence for a real valued
parameter $z$ is to exhibit a real-valued random process $E_t(z)$ which,
when evaluated at the true value $z^*$ of the parameter of interest,
is a non-negative supermartingale with initial value of 1, in which case
Ville's inequality ensures $\mathbb{P}\left(\forall t \in \mathbb{N}:
E_t(z^*) \leq \alpha^{-1}\right) \geq 1 - \alpha$.  If the process
$E_t(z)$ is monotonically increasing in $z$, then the supremum of the
lower contour set $L_t \doteq \sup_z \left\{ z | E_t(z) \leq \alpha^{-1} \right\}$
is suitable as a lower confidence sequence; an upper confidence sequence
can be analogously defined.

We use the above strategy.  First we lower bound \cref{eqn:binmart},
\begin{equation}
\eqnbinmart
\tag{\ref{eqn:binmart}}
\end{equation}
and eliminate the explicit dependence upon $\theta_s$, by noting
$h(\lambda, \cdot)$ is concave and therefore
\begin{equation}
E_t(\lambda) \geq \exp\left(\lambda t \left(q_t - \hat{q}_t\right) - t\ h\left(\lambda, q_t\right) \right),
\label{eqn:binmartlb}
\end{equation}
because
$\left( t f(q) = \max_{\theta \bigl| 1^\top \theta=t q} \sum_{s \leq t} f(\theta_s) \right)$
for any concave $f$. \cref{eqn:binmartlb} is monotonically increasing
in $q_t$ and therefore defines a lower confidence sequence.  For an
upper confidence sequence we use $q_t = 1 - (1 - q_t)$ and a lower
confidence sequence on $(1 - q_t)$.

Regarding the choice of $\lambda$, in practice many $\lambda$ are
(implicitly) used via stitching (i.e., using different $\lambda$ in
different time epochs and majorizing the resulting bound in closed form)
or mixing (i.e., using a particular fixed mixture of \cref{eqn:binmartlb}
via a discrete sum or continuous integral over $\lambda$); our choices
will depend upon whether we are designing for tight asymptotic rates or
low computational footprint.  We provide specific details associated
with each theorem or experiment.

Note \cref{eqn:binmartlb} is invariant to permutations of $X_{1:t}$
and hence the empirical CDF at time $t$ is a sufficient statistic for
calculating \cref{eqn:binmartlb} at any $v$.

\subsection{Challenge with quantile space}
\label{app:whynotquantilespace}
In this section assume all CDFs are invertible for ease of exposition.

In the i.i.d. setting, \cref{eqn:binmart} can be evaluated at the
(unknown) fixed $v(q)$ which corresponds to quantile $q$.  Without
knowledge of the values, one can assert the existence of such values for a
countably infinite collection of quantiles and a careful union bound of
Ville's inequality on a particular discretization can yield an LIL rate:
this is the approach of \citet{howard2022sequential}.  A key
advantage of this approach is covariance to monotonic transformations.

Beyond the i.i.d. setting, one might hope to analogously evaluate
\cref{eqn:binmart} at an unknown fixed value $v_t(q)$ which for each $t$
corresponds to quantile $q$.  Unfortunately, $v_t(q)$ is not just unknown,
but also unpredictable with respect to the initial filtration,
and the derivation that \cref{eqn:binmart} is a martingale depends
upon $v$ being predictable.  In the case that $X_t$ is independent
but not identically distributed, $v_t(q)$ is initially predictable
and therefore this approach could work, but would only be valid under
this assumption.

The above argument does not completely foreclose the possibility of
a quantile space approach, but merely serves to explain why the
authors pursued a value space approach in this work.  We encourage
the interested reader to innovate.

\section{Unit Interval Bounds}
\label{app:boundzeroone}

\subsection{Lower Bound}
\label{app:lowerboundzeroone}

\begin{algorithm}[tb]
     \caption{Unit Interval Lower Bound.  $\epsilon(d)$ is an increasing function specifying the resolution of discretization at level $d$.  $\lboracle_t\left(\rho; \delta, d, \suffstat_t\right)$ is a lower confidence sequence for fixed value $\rho$ with coverage at least $\left(1 - \delta\right)$.}
     \label{alg:lbunionzeroone}
  \begin{algorithmic}
     \STATE {\bfseries Input:} value $v$; confidence $\alpha$; sufficient statistic $\suffstat_t$. \hfill\algcommentlight{comments below indicate differences from upper bound}
     \STATE \algcommentlight{$\suffstat_t \doteq X_{1:t}$ or $\suffstat_t \doteq (W_{1:t}, X_{1:t})$}
     \STATE {\bfseries Output:} $L_t(v)$ satisfying \cref{eqn:overallguarantee}.
     \IIF{$v < 0$} \textbf{return} 0 \ENDIIF \hfill \algcommentlight{check for underflow of range rather than overflow}
     \STATE $l \leftarrow 0$ \hfill \algcommentlight{initialize with 0 instead of 1}
     \STATE $v \leftarrow \min\left(1, v\right)$ \hfill \algcommentlight{project onto [0, 1] using $\min$ instead of $\max$}
     \FOR{$d=1$ {\bfseries to} $\infty$}
     \STATE $\rho_d \leftarrow \epsilon(d) \lfloor \epsilon(d)^{-1} v \rfloor$ \hfill \algcommentlight{use floor instead of ceiling}
     \STATE $\delta_d \leftarrow \nicefrac{\alpha}{2^d \epsilon(d)}$
     \STATE $l \leftarrow \max\left(l, \lboracle_t\left(\rho_d; \delta, \suffstat_t\right)\right)$ \hfill \algcommentlight{use lower bound instead of upper bound}
     \IF{$0 = \sum_{s \leq t} 1_{X_s \in \left[\rho_d, v\right)}$}
        \RETURN $l$
     \ENDIF
     \ENDFOR
  \end{algorithmic}
  \end{algorithm}

\cref{alg:lbunionzeroone} is extremely similar to
\cref{alg:ubunionzeroone}: the differences are indicated in comments.
Careful inspection reveals the output of \cref{alg:ubunionzeroone},
$U_t(v)$, can be obtained from the output of \cref{alg:lbunionzeroone},
$L_t(v)$, via $U_t(v) = 1 - L_t(1 - v)$; but only if the sufficient
statistics are adjusted such that
$\uboracle_t(\rho_d; \delta, \suffstat_t) = 1 - \lboracle_t(1 - \rho_d; \delta, \suffstat_t')$.
The reference implementation uses this strategy.

\subsection{Proof of \cref{thm:coverage}}
\label{app:proofthmcoverage}

We prove the results for the upper bound \cref{alg:ubunionzeroone}; the argument for the lower bound \cref{alg:lbunionzeroone} is similar.

The algorithm terminates when we find a $d$ such that $0 = \sum_{s \leq t} 1_{X_s \in (v, \rho_d]}$. Since $\epsilon(d) \uparrow \infty$ as $d \uparrow \infty$, we have $\rho_d = \epsilon(d) \lceil \epsilon(d)^{-1} v \rceil \downarrow v$, so that $\sum_{s \leq t} 1_{X_s \in (v, \rho_d]} \downarrow 0$. So the algorithm must terminate.

At level $d$, we have $\epsilon(d)$ confidence sequences. The $i$\textsuperscript{th} confidence sequence at level $d$ satisfies
\begin{align}
P(\exists t: \AvCDF_t(v) > \uboracle_t(i / \epsilon(d); \delta_d, d, \suffstat_t))
&\leq \frac{\alpha}{2^d \epsilon(d)}.
\end{align}
Taking a union bound over all confidence sequences at all levels, we have
\begin{align}
P\left(
    \exists d \in \mathbb{N}, i \in \lbrace 1, \dots, d \rbrace,
        t \in \mathbb{N}:
    \AvCDF_t(i / \epsilon(d))
        > \uboracle_t(i / \epsilon(d); \delta, d, \suffstat_t)
\right) \leq \alpha.
\end{align}

Thus we are assured that, for any $v \in \mathbb{R}$,
\begin{align}
P(\forall t, d: \AvCDF_t(v) \leq \AvCDF_t(\rho_d)
    \leq \uboracle_t(\rho_d; \delta_d, d, \suffstat_t))
\geq 1 - \alpha.
\end{align}
\Cref{alg:ubunionzeroone} will return $\uboracle_t(\rho_d; \delta_d, d, \suffstat_t)$ for some $d$ unless all such values are larger than one, in which case it returns the trivial upper bound of one. This proves the upper-bound half of guarantee \eqref{eqn:overallguarantee}. A similar argument proves the lower-bound half, and union bound over the upper and lower bounds finishes the argument.

\section{Proof of \cref{thm:smoothedregretzeroone}}
\thmSmooth*

\label{app:proofsmoothedregretzeroone}

Note $v$ is fixed for the entire argument below, and $\xi_t$ denotes
the unknown smoothness parameter at time $t$.

We will argue that the upper confidence radius $U_t(v) - t^{-1} \sum_{s
\leq t} 1_{X_s \leq v}$ has the desired rate. An analogous argument
applies to the lower confidence radius $t^{-1} \sum_{s \leq t} 1_{X_s
\leq v} - L_t(v)$, and the confidence width $U_t(v) - L_t(v)$ is the
sum of these two.

For the proof we introduce an integer parameter $\eta \geq 2$ which
controls both the grid spacing ($\epsilon(d) = \eta^d$) and the
allocation of error probabilities to levels ($\delta_d = \alpha /
(\eta^d \epsilon(d))$). In the main paper we set $\eta = 2$.

At level $d$ we construct $\eta^d$ confidence sequences on an
evenly-spaced grid of values $1/\eta^d, 2/\eta^d, \dots, 1$. We divide
total error probability $\alpha / \eta^d$ at level $d$ among these
$\eta^d$ confidence sequences, so that each individual confidence sequence
has error probability $\alpha/\eta^{2d}$.

For a fixed bet $\lambda$ and value $\rho$, $S_t$ defined in
\cref{subsec:unitinterval} is sub-Bernoulli qua \citet[Definition
1]{howard2021time} and therefore sub-Gaussian with variance process
$V_t \doteq t K(q_t)$, where $K(p) \doteq \nicefrac{(2 p - 1)}{2
\log\left(\nicefrac{p}{1-p}\right)}$ is from \citet{kearns1998large};
from \citet[Proposition 5]{howard2021time} it follows that there exists
an explicit mixture distribution over $\lambda$ such that
\begin{align}
M(t; q_t, \tau) &\doteq \sqrt{2 \left(t K(q_t) + \tau\right) \log\left(\frac{\eta^{2d}}{2 \alpha} \sqrt{\frac{t K(q_t) + \tau}{\tau}} + 1\right)}
\label{eqn:binnormalcurved}
\end{align}
is a (curved) uniform crossing boundary, i.e., satisfies
\begin{align*}
\frac{\alpha}{\eta^{2d}} &\geq \mathbb{P}\left(\exists t \geq 1: S_t \geq \frac{M(t; q_t, \tau)}{t} \right),
\end{align*}
where $S_t \doteq \AvCDF_t(\rho) - t^{-1} \sum_{s \leq t} 1_{X_s \leq
\rho}$ is from \cref{eqn:binmart}, and $\tau$ is a hyperparameter
to be determined further below.

Because the values at level $d$ are $1/\eta^d$ apart, the worst-case
discretization error in the estimated average CDF value is
\begin{align*}
\AvCDF_t(\epsilon(d) \lceil \epsilon(d)^{-1} v \rceil) - \AvCDF_t(v)
\leq 1/(\xi_t \eta^d),
\end{align*}
and the total worst-case confidence radius including discretization error is
\begin{align*}
r_d(t) &= \frac{1}{\xi_t \eta^d} + \sqrt{\frac{2 \left(K(q_t) + \tau/t\right)}{t} \log\left(\frac{\eta^{2d}}{2 \alpha} \sqrt{\frac{t K(q_t) + \tau}{\tau}} + 1\right)}.
\end{align*}
Now evaluate at $d$ such that
$\sqrt{\psi_t} < \xi_t \eta^d \leq \eta \sqrt{\psi_t}$ where $
\psi_t \doteq t \left(K(q_t) + \tau/t\right)^{-1}$,
\begin{align*}
r_d(t) &\leq \sqrt{\frac{K(q_t) + \tau/t}{t}} + \sqrt{\frac{2 \left(K(q_t) + \tau/t\right)}{t} \log\left(\frac{\xi_t^{-2} \eta^2}{2 \alpha} \left(\frac{t}{K(q_t) + \tau/t}\right) \sqrt{\frac{t K(q_t) + \tau}{\tau}} + 1\right)}.
\end{align*}
The final result is not very sensitive to the choice of $\tau$, and we
use $\tau = 1$ in practice.

\section{Extensions}

\subsection{Arbitrary Support}
\label{app:arbitrarysupport}

\begin{algorithm}[tb]
   \caption{Entire Real Line Upper Bound.  $\epsilon(d)$ is an increasing function specifying the resolution of discretization at level $d$.  $\uboracle_t\left(\rho; \delta, d, \suffstat_t\right)$ is an upper confidence sequence for fixed value $\rho$ with coverage at least $\left(1 - \delta\right)$.}
   \label{alg:ubunionrealline}
\begin{algorithmic}
   \STATE {\bfseries Input:} value $v$; confidence $\alpha$; sufficient statistic $\suffstat_t$.
   \STATE \algcommentlight{e.g. $\suffstat_t \doteq X_{1:t}$ or $\suffstat_t \doteq (W_{1:t}, X_{1:t})$}
   \STATE {\bfseries Output:} $U_t(v)$ satisfying \cref{eqn:overallguarantee}.
   \STATE $u \leftarrow 1$
   \FOR{$d=1$ {\bfseries to} $\infty$}
   \STATE $k_d \leftarrow \lceil \epsilon(d)^{-1} v \rceil$ \hfill \algcommentlight{Sub-optimal: see text for details}
   \STATE $\rho_d \leftarrow \epsilon(d) k_d$
   \STATE $\delta_d \leftarrow \left(\nicefrac{\alpha}{2^d}\right) \left(\nicefrac{3}{(\pi^2 - 3) (1 + |k_d|)^2}\right)$ \hfill \algcommentlight{Union bound over $d \in \mathbb{N}$ and $k_d \in \mathbb{Z}$}
   \STATE $u \leftarrow \min\left(u, \uboracle_t\left(\rho_d; \delta_d, d, \suffstat_t\right)\right)$
   \IF{$0 = \sum_{s \leq t} 1_{X_s \in \left(v, \rho_d\right]}$}
      \RETURN $u$
   \ENDIF
   \ENDFOR
\end{algorithmic}
\end{algorithm}

\cref{alg:ubunionrealline} is a variation on \cref{alg:ubunionzeroone} which
does not assume a bounded range, and instead uses a countably discrete
dense subset of the entire real line.  Using the same argument
of \cref{thm:smoothedregretzeroone} with the modified probability
from the modified union bound, we have
\begin{align*}
|k_d| - 1 &< \eta^{-d} |v| \leq |k_d|, \\
\xi_t / \sqrt{\psi_t} &> \eta^{-d} \geq \eta^{-1} \xi_t / \sqrt{\psi_t} \\
\implies 1 + |k_d| &< 2 + \xi_t |v| / \sqrt{\psi_t} \\
\implies r_d(t) &\leq 
 \logvwidthbound,
\end{align*}
demonstrating a logarithmic penalty in the probe value $v$ (e.g., \cref{fig:lognormalwidth}).

\paragraph{Sub-optimality of $k_d$} The choice of $k_d$
in \cref{alg:ubunionrealline} is amenable to analysis,
but unlike in \cref{alg:ubunionzeroone}, it is not optimal.
In \cref{alg:ubunionzeroone} the probability is allocated uniformly at
each depth, and therefore the closest grid point provides the tightest
estimate.  However in \cref{alg:ubunionrealline}, the probability
budget decreases with $|k_d|$ and because $k_d$ can be negative, it
is possible that a different $k_d$ can produce a tighter upper bound.
Since every $k_d$ is covered by the union bound, in principle we could
optimize over all $k_d$ but it is unclear how to do this efficiently.
In our implementation we do not search over all $k_d$, but we do adjust
$k_d$ to be closest to the origin with the same empirical counts.

\subsection{Discrete Jumps}
\label{app:discretejumps}

\paragraph{Known Countably Infinite}
Suppose $D$ is smooth wrt a reference measure $M$, where $M$ is of
the form $$M = \breve{M} + \sum_{i \in I} \zeta_i 1_{v_i},$$ with $I$
a countable index set, $1 \geq \sum_{i \in I} \zeta_i$ and $\breve{M}$
a sub-probability measure normalizing to $(1 - \sum_{i \in I} \zeta_i)$.
Then we can allocate $(1 - \sum_{i \in I} \zeta_i)$ of our overall
coverage probability to bounding $\breve{M}$ using \cref{alg:ubunionzeroone}
and \cref{alg:lbunionzeroone}.  For the remaining $\{ v_i \}_{i \in I}$
we can run explicit pointwise bounds each with coverage probability
fraction $\zeta_i$.

Computationally, early termination of the infinite search over the
discrete bounds is possible.  Suppose (wlog) $I$ indexes $\zeta$ in
non-increasing order, i.e., $i \leq j \implies \zeta_i \leq \zeta_j$: then as soon as there are no remaining empirical
counts between the desired value $v$ and the most recent discrete value
$v_i$, the search over discrete bounds can terminate.

\section{Importance-Weighted Variant}
\label{app:importanceweighted}

\subsection{Modified Bounds}

\cref{alg:ubunionzeroone} and \cref{alg:lbunionzeroone} are unmodified,
with the caveat that the oracles $\lboracle_t$ and $\uboracle_t$ must
now operate on an importance-weighted realization $(W_{1:t}, X_{1:t})$,
rather then directly on the realization $X_{1:t}$.

\subsubsection{DDRM Variant}

For simplicity we describe the lower bound $\lboracle_t$ only.  The
upper bound is derived analogously via the equality
$Y_s = W_s - (W_s - Y_s)$ and a lower bound on $(W_s - Y_s)$: see
\citet[Remark 3]{waudby2022anytime} for more details.

This is the Heavy NSM from \citet{mineiro2022lower} combined with the
$L^*$ bound of \citet[\S4.2.3]{orabona2019modern}.  The Heavy NSM allow us
to handle importance weights with unbounded variance, while the Adagrad
$L^*$ bound facilitates lazy evaluation.

For fixed $v$, let $Y_t = W_t 1_{X_t \geq v}$ be a non-negative
real-valued discrete-time random process, let $\hat{Y}_t \in [0, 1]$ be a
predictable sequence, and let $\lambda \in [0, 1)$ be a fixed scalar bet.
Then $$
\begin{aligned}
E_t(\lambda) &\doteq \exp\left( \lambda \left(\sum_{s \leq t} \hat{Y}_s - \mathbb{E}_{s-1}\left[Y_s\right]\right) + \sum_{s \leq t} \log\left(1 + \lambda \left(Y_s - \hat{Y}_s\right) \right) \right) \\
\end{aligned}
$$ is a test supermartingale~\citep[\S3]{mineiro2022lower}.  Manipulating, $$
\begin{aligned}
E_t(\lambda) &= \exp\left( \lambda \left(\sum_{s \leq t} Y_s - \mathbb{E}_{s-1}\left[Y_s\right]\right) - \sum_{s \leq t} \underbrace{\left( \lambda \left(Y_s - \hat{Y}_s\right) - \log\left(1 + \lambda \left(Y_s - \hat{Y}_s\right) \right) \right)}_{\doteq h\left(\lambda \left(Y_s - \hat{Y}_s\right)\right)} \right) \\
&= \exp\left( \lambda \left(\sum_{s \leq t} Y_s - \mathbb{E}_{s-1}\left[Y_s\right]\right) - \sum_{s \leq t} h\left(\lambda \left(Y_s - \hat{Y}_s\right)\right) \right) \\
&\geq \exp\left( \lambda \left(\sum_{s \leq t} Y_s - \mathbb{E}_{s-1}\left[Y_s\right]\right) - \left( \sum_{s \leq t} h\left(\lambda \left(Y_s - \hat{Y}_t^*\right)\right) \right) - \text{Reg}(t) \right) & \left(\dagger\right) \\
&= \exp\left( \lambda \left(t \hat{Y}_t^* - \sum_{s \leq t} \mathbb{E}_{s-1}\left[Y_s\right]\right) + \sum_{s \leq t} \log\left(1 + \lambda \left(Y_s - \hat{Y}_t^*\right) \right) - \text{Reg}(t) \right),
\end{aligned}
$$ where for $(\dagger)$ we use a no-regret learner
on $h()$ with regret $\text{Reg}(t)$ to any constant
prediction $\hat{Y}_t^* \in [0, 1]$.  The function $h()$ is
$M$-smooth with $M = \frac{\lambda^2}{(1 - \lambda)^2}$
so we can get an $L^*$ bound~\citep[\S4.2.3]{orabona2019modern} of
$$
\begin{aligned}
\text{Reg}(t) &= 4 \frac{\lambda^2}{(1 - \lambda)^2} + 4 \frac{\lambda}{1 - \lambda} \sqrt{\sum_{s \leq t} h\left(\lambda \left(Y_s - \hat{Y}_t^*\right)\right)} \\
&= 4 \frac{\lambda^2}{(1 - \lambda)^2} + 4 \frac{\lambda}{1 - \lambda} \sqrt{\left(-t \hat{Y}_t^* + \sum_{s \leq t} Y_s\right) - \sum_{s \leq t} \log\left(1 + \lambda \left(Y_s - \hat{Y}_t^*\right)\right)},
\end{aligned}
$$ thus essentially our variance process is inflated by a square-root.  In
exchange we do not have to actually run the no-regret algorithm, which
eases the computational burden.  We can compete with any in-hindsight
prediction: if we choose to compete with the clipped running mean
$\overline{Y_t}$ then we end up with
\begin{equation}
E_t(\lambda) \geq \exp\left( \lambda \left(\min\left(t, \sum_{s \leq t} Y_s\right) - \mathbb{E}_{s-1}\left[Y_s\right]\right) + \sum_{s \leq t} \log\left(1 + \lambda \left(Y_s - \overline{Y_t}\right) \right)  - \text{Reg}(t) \right),
\label{eqn:heavynsmplusadagrad}
\end{equation}
which is implemented in the reference implementation as \pythonmethod{LogApprox:getLowerBoundWithRegret(lam)}.  The $\lambda$-s are mixed
using DDRM from \citet[Thm. 4]{mineiro2022lower},
implemented via the \pythonmethod{DDRM} class and the
\pythonmethod{getDDRMCSLowerBound} method in the reference implementation.
\pythonmethod{getDDRMCSLowerBound} provably correctly early terminates
the infinite sum by leveraging
$$
\begin{aligned}
\sum_{s \leq t} \log\left(1 + \lambda \left(Y_s - \overline{Y_t}\right) \right) &\leq \lambda \left(\sum_{s \leq t} Y_s - t \overline{Y_t}\right)
\end{aligned}
$$ as seen in the termination criterion of the inner method \pythonmethod{logwealth(mu)}.

To minimize computational overhead, we can lower bound
$\log(a + b)$ for $b \geq 0$
using strong concavity qua \citet[Thm. 3]{mineiro2022lower},
resulting in the following geometrically spaced collection of sufficient
statistics:
$$
\begin{aligned}
(1 + k)^{n_l} = z_l &\leq z < z_u = (1 + k) z_l = (1 + k)^{n_l+1}, \\
\end{aligned}
$$ along with distinct statistics for $z = 0$.  $k$ is a hyperparameter
controlling the granularity of the discretization (tighter lower bound vs.
more space overhead): we use $k = 1/4$ exclusively in our experiments.
Note the coverage guarantee is preserved for any choice of $k$ since we
are lower bounding the wealth.

Given these statistics, the wealth can be lower bounded given any
bet $\lambda$ and any in-hindsight prediction $\hat{Y}_t^*$ via $$
\begin{aligned}
f(z) &\doteq \log\left(1 + \lambda \left(z - \hat{Y}_t^*\right) \right), \\
f(z) &\geq \alpha f(z_l) + (1 - \alpha) f(z_u) + \frac{1}{2} \alpha (1 - \alpha) m(z_l), \\
\alpha &\doteq \frac{z_u - z}{z_u - z_l}, \\
m(z_l) &\doteq \left( \frac{k z_l \lambda}{k z_l \lambda + 1 - \lambda \hat{Y}_t^*} \right)^2.
\end{aligned}
$$ Thus when accumulating the statistics, for each $Y_s = W_s 1_{X_s
\geq v}$, a value of $\alpha$ must be accumulated at key $f(z_l)$,
a value of $(1 - \alpha)$ accumulated at key $f(z_u)$, and a value of
$\alpha (1 - \alpha)$ accumulated at key $m(z_l)$.
The \pythonmethod{LogApprox::update} method from the reference implementation
implements this.

Because these sufficient statistics are data linear, a further
computational trick is to accumulate the sufficient statistics with
equality only, i.e., for $Y_s = W_s 1_{X_s = v}$; and when
the CDF curve is desired, combine these point statistics into
cumulative statistics.  In this manner only $O(1)$ incremental
work is done per datapoint; while an additional $O(t \log(t))$ work is done
to accumulate all the sufficient statistics only when the bounds
need be computed.  The method
\pythonmethod{StreamingDDRMECDF::Frozen::\_\_init\_\_} from the reference
implementation contains this logic.

\subsubsection{Empirical Bernstein Variant}
\label{subsec:empirical_bernstein_variant}

For simplicity we describe the lower bound $\lboracle_t$ only.  The
upper bound is derived analogously via the equality
$Y_s = W_s - (W_s - Y_s)$ and a lower bound on $(W_s - Y_s)$: see
\citet[Remark 3]{waudby2022anytime} for more details.

This is the empirical Bernstein NSM from \citet{howard2021time} combined
with the $L^*$ bound of \citet[\S4.2.3]{orabona2019modern}.  Relative to
DDRM it is faster to compute, has a more concise sufficient statistic,
and is easier to analyze; but it is wider empirically, and theoretically
requires finite importance weight variance to converge.

For fixed $v$, let $Y_t = W_t 1_{X_t \geq v}$ be a non-negative
real-valued discrete-time random process, let $\hat{Y}_t \in [0, 1]$ be a
predictable sequence, and let $\lambda \in [0, 1)$ be a fixed scalar bet.
\defcitealias{fan2015exponential}{Fan}
Then $$
\begin{aligned}
E_t(\lambda) &\doteq \exp\left( \lambda \left(\sum_{s \leq t} \hat{Y}_s - \mathbb{E}_{s-1}\left[Y_s\right]\right) + \sum_{s \leq t} \log\left(1 + \lambda \left(Y_s - \hat{Y}_s\right) \right) \right) \\
\end{aligned}
$$ is a test supermartingale~\citep[\S3]{mineiro2022lower}.  Manipulating, $$
\begin{aligned}
E_t(\lambda) &\doteq \exp\left(
\lambda \left(\sum_{s \leq t} Y_s - \mathbb{E}_{s-1}\left[Y_s\right]\right)
- \sum_{s \leq t} \underbrace{\left( \lambda \left(Y_s - \hat{Y}_s\right) - \log\left(1 + \lambda \left(Y_s - \hat{Y}_s\right) \right) \right)}_{\doteq h\left(\lambda \left(Y_s - \hat{Y}_s\right)\right)}
\right) \\
&\geq \exp\left(
\lambda \left(\sum_{s \leq t} Y_s - \mathbb{E}_{s-1}\left[Y_s\right]\right) - h(-\lambda) \sum_{s \leq t} \left(Y_s - \hat{Y}_s\right)^2\right) & \text{\citepalias[Lemma 4.1]{fan2015exponential}} \\
&\geq \exp\left(
\lambda \left(\sum_{s \leq t} Y_s - \mathbb{E}_{s-1}\left[Y_s\right]\right) - h(-\lambda) \left( \text{Reg}(t) + \sum_{s \leq t} \left(Y_s - Y^*_t\right)^2\right) \right) & \left(\dagger\right), \\
&\doteq \exp\left(\lambda S_t - h(-\lambda) V_t\right), \label{eq:empbern_supermg}
\end{aligned}
$$ where $S_t = \sum_{s \leq t} Y_s - \mathbb{E}_{s-1}\left[Y_s\right]$ and for $(\dagger)$ we use a no-regret learner on squared loss on feasible set $[0, 1]$ with regret $\text{Reg}(t)$ to any constant in-hindsight prediction $\hat{Y}_t^* \in [0, 1]$. Since $Y_s$ is unbounded above, the loss is not Lipschitz and we can't get fast rates for squared loss, but we can run Adagrad and get an $L^*$ bound, $$
\begin{aligned}
\text{Reg}(t) &= 2 \sqrt{2} \sqrt{\sum_{s \leq t} g_s^2} \\
&= 4 \sqrt{2} \sqrt{\sum_{s \leq t} (Y_s - \hat{Y}_s)^2} \\
&\leq 4 \sqrt{2} \sqrt{\text{Reg}(t) + \sum_{s \leq t} (Y_s - \hat{Y}_t^*)^2}, \\
\implies \text{Reg}(t) &\leq 16 + 4 \sqrt{2} \sqrt{8 + \sum_{s \leq t} (Y_s - \hat{Y}_t^*)^2}.
\end{aligned}
$$ Thus basically our variance process is inflated by an additive square root.

We will compete with $Y^*_t = \min\left(1, \frac{1}{t} \sum_s Y_s\right)$.

A key advantage of the empirical Bernstein over DDRM is the availability
of both a conjugate (closed-form) mixture over $\lambda$ and a closed-form
majorized stitched boundary.  This yields both computational speedup and
analytical tractability.

For a conjugate mixture, we use the truncated gamma prior from
\citet[Theorem 2]{waudby2022anytime} which yields mixture wealth
\newcommand{\eqniwbernada}{
M_t^{\text{EB}} \doteq \left(\frac{\tau^\tau e^{-\tau}}{\Gamma(\tau) - \Gamma(\tau, \tau)}\right) \left( \frac{1}{\tau + V_t} \right) {}_1F_1\left(1, V_t + \tau + 1, S_t + V_t + \tau\right)
}
\begin{equation}
\eqniwbernada,
\label{eqn:iwempbernada}
\end{equation}
where ${}_1F_1(\ldots)$ is Kummer's confluent hypergeometric function
and $\Gamma(\cdot, \cdot)$ is the upper incomplete gamma function.
For the hyperparameter, we use $\tau = 1$.

\subsection{Proof of \cref{thm:iwsmoothedregretzeroone}}
\label{app:proofiwsmoothedregretzeroone}

\thmIwSmooth*

Note $v$ is fixed for the entire argument below, and $\xi_t$ denotes
the unknown smoothness parameter at time $t$.

We will argue that the upper confidence radius $U_t(v) - t^{-1}
\sum_{s \leq t} W_s 1_{X_s \leq v}$ has the desired rate. An analogous
argument applies to the lower confidence radius.  One difference
from the non-importance-weighted case is that, to be sub-exponential,
the lower bound is constructed from an upper bound on $U'_t(v) = W_s
(1 - 1_{X_s \leq v})$ via $L_t(v) - 1 - U'_t(v)$, which introduces an
additional $B_t = t^{-1} \sum_{s \leq t} (W_s - 1)$ term to the width.
(Note, because $\forall t: \mathbb{E}_t[W_t - 1] = 0$, this term will
concentrate, but we will simply use the realized value here.)

For the proof we introduce an integer parameter $\eta \geq 2$ which
controls both the grid spacing ($\epsilon(d) = \eta^d$) and the
allocation of error probabilities to levels ($\delta_d = \alpha /
(\eta^d \epsilon(d))$). In the main paper we set $\eta = 2$.

At level $d$ we construct $\eta^d$ confidence sequences on an
evenly-spaced grid of values $1/\eta^d, 2/\eta^d, \dots, 1$. We divide
total error probability $\alpha / \eta^d$ at level $d$ among these
$\eta^d$ confidence sequences, so that each individual confidence sequence
has error probability $\alpha/\eta^{2d}$.

For a fixed bet $\lambda$ and value $\rho$, $S_t$ defined in
\cref{subsec:empirical_bernstein_variant} is sub-exponential
qua \citet[Definition 1]{howard2021time} and therefore
from \cref{lemma:iwcurved} there exists an explicit mixture
distribution over $\lambda$ inducing (curved) boundary
\begin{align}
\frac{\alpha}{\eta^{2d}} &\geq \mathbb{P}\left( \exists t \geq 1 : \frac{S_t}{t} \geq \max\left(\frac{C(\tau)}{t}, u\left(V_t; \tau, \frac{\alpha}{\eta^{2d}}\right) \right) \right), \nonumber \\
u\left(V_t; \tau, \frac{\alpha}{\eta^{2d}}\right) &= \sqrt{2 \left(\frac{(\tau + V_t)/t}{t}\right) \log\left( \sqrt{\frac{\tau + V_t}{2 \pi}} e^{-\frac{1}{12 (\tau + V_t) + 1}} \left( \frac{1 + \eta^{2d} \alpha^{-1}}{C(\tau)} \right) \right) } \nonumber \\
&\qquad + \frac{1}{t} \log\left(\sqrt{\frac{\tau + V_t}{2 \pi}} e^{-\frac{1}{12 (\tau + V_t) + 1}} \left( \frac{1 + \eta^{2d} \alpha^{-1}}{C(\tau)} \right)\right), \label{eqn:iwcurvedboundary}
\end{align}
where $S_t \doteq \AvCDF_t(\rho) - t^{-1} \sum_{s \leq t} W_s 1_{X_s \leq
\rho}$, and $\tau$ is a hyperparameter to be determined further below.

Because the values at level $d$ are $1/\eta^d$ apart, the worst-case
discretization error in the estimated average CDF value is
\begin{align*}
\AvCDF_t(\epsilon(d) \lceil \epsilon(d)^{-1} v \rceil) - \AvCDF_t(v)
\leq 1/(\xi_t \eta^d),
\end{align*}
and the total worst-case confidence radius including discretization error is
\begin{align*}
r_d(t) &= \frac{1}{\xi_t \eta^d} + \max\left(\frac{C(\tau)}{t}, u\left(V_t; \tau, \frac{\alpha}{\eta^{2d}}\right)\right).
\end{align*}
Now evaluate at $d$ such that $\sqrt{\psi_t} < \xi_t \eta^d \leq \eta \sqrt{\psi_t}$ where $
\psi_t \doteq t \left((\tau + V_t)/t\right)^{-1}$,
\begin{align*}
r_d(t) &\leq \frac{1}{\sqrt{\psi_t}} + \max\left(\frac{C(\tau)}{t}, u\left(V_t; \tau, \frac{\alpha}{\eta^2 \xi_t^{-2} \psi_t}\right)\right) \\
&= \sqrt{\frac{(\tau + V_t)/t}{t}} + \tilde{O}\left(\sqrt{\frac{(\tau + V_t)/t}{t} \log\left(\xi_t^{-2} \alpha^{-1}\right)}\right) + \tilde{O}(t^{-1} \log\left(\xi_t^{-2} \alpha^{-1}\right)),
\end{align*}
where $\tilde{O}()$ elides polylog $V_t$ factors.  The final result
is not very sensitive to the choice of $\tau$, and we use $\tau = 1$
in practice.

\begin{lemma}
\label{lemma:iwcurved}
Suppose
\begin{align*}
&\exp\left( \lambda S_t - \psi_e(\lambda) V_t \right), \\
\psi_e(\lambda) &\doteq -\lambda - \log(1 - \lambda),
\end{align*}
is sub-$\psi_e$ qua \citet[Definition 1]{howard2021time}; then there
exists an explicit mixture distribution over $\lambda$ with hyperparameter
$\tau > 0$ such that
\begin{align*}
\alpha &\geq \mathbb{P}\left( \exists t \geq 1 : \frac{S_t}{t} \geq \max\left(\frac{C(\tau)}{t}, u\left(V_t; \tau, \alpha\right) \right) \right), \\
u\left(V_t; \tau, \alpha\right) &= \sqrt{2 \left(\frac{(\tau + V_t)/t}{t}\right) \log\left( \sqrt{\frac{\tau + V_t}{2 \pi}} e^{-\frac{1}{12 (\tau + V_t) + 1}} \left( \frac{1 + \alpha^{-1}}{C(\tau)} \right) \right) } \\
&\qquad + \frac{1}{t} \log\left(\sqrt{\frac{\tau + V_t}{2 \pi}} e^{-\frac{1}{12 (\tau + V_t) + 1}} \left( \frac{1 + \alpha^{-1}}{C(\tau)} \right)\right), \\
C(\tau) &\doteq \frac{\tau^\tau e^{-\tau}}{\Gamma(\tau) - \Gamma(\tau, \tau)},
\end{align*}
is a (curved) uniform crossing boundary.
\begin{proof}
We can form the conjugate mixture using a truncated gamma prior from
\citet[Proposition 9]{howard2021time}, in the form from \citet[Theorem
2]{waudby2022anytime}, which is our \cref{eqn:iwempbernada}.
\begin{equation*}
\eqniwbernada,
\end{equation*}
where ${}_1F_1(\ldots)$ is Kummer's confluent hypergeometric function.
Using \citet[identity 13.6.5]{olver2010nist},
\begin{align*}
{}_1F_1(1,a+1,x) &= e^x a x^{-a} \left(\Gamma(a) - \Gamma(a, x)\right)
\end{align*}
where $\Gamma(a, x)$ is the (unregularized) upper incomplete gamma function.
From \citet[Theorem 1.2]{pinelis2020exact} we have
\begin{align*}
\Gamma(a, x) &< \frac{x^a e^{-x}}{x - a} \\
\implies {}_1F_1(1,a+1,x) &\geq e^x a x^{-a} \Gamma(a) - \frac{a}{x - a}.
\end{align*}
Applying this to the mixture yields
\begin{align*}
M_t^{\text{EB}} %
&\geq \frac{C(\tau) e^{\tau + V_t + S_t}}{(\tau + V_t + S_t)^{\tau + V_t}} \Gamma\left(\tau + V_t\right) - \frac{C(\tau)}{S_t} \\
&\geq \frac{C(\tau) e^{\tau + V_t + S_t}}{(\tau + V_t + S_t)^{\tau + V_t}} \Gamma\left(\tau + V_t\right) - 1, & \left(\dagger\right)
\end{align*}
where $\left(\dagger\right)$ follows from the self-imposed
constraint $S_t \geq C(\tau)$.
This yields crossing boundary
\begin{align*}
\alpha^{-1} &= \frac{C(\tau) e^{\tau + V_t + S_t}}{(\tau + V_t + S_t)^{\tau + V_t}} \Gamma\left(\tau + V_t\right) - 1, \\
\frac{e^{\tau + V_t + S_t}}{\left(1 + \frac{S_t}{\tau + V_t}\right)^{\tau + V_t}} &= \left( \frac{\left(\tau + V_t\right)^{\tau + V_t}}{\Gamma\left(\tau + V_t\right)}\right) \left(\frac{1 + \alpha^{-1}}{C(\tau)} \right) \doteq \left( \frac{\left(\tau + V_t\right)^{\tau + V_t}}{\Gamma\left(\tau + V_t\right)}\right) \phi_t(\tau, \alpha), \\
\frac{e^{1 + \frac{S_t}{\tau + V_t}}}{\left(1 + \frac{S_t}{\tau + V_t}\right)} &= \left( \frac{\left(\tau + V_t\right)^{\tau + V_t}}{\Gamma\left(\tau + V_t\right)}\right)^{\frac{1}{\tau + V_t}} \phi_t(\tau, \alpha)^{\frac{1}{\tau + V_t}} \doteq z_t, \\
S_t &= \left(\tau + V_t\right) \left( -1 - W_{-1}\left(-z_t^{-1}\right) \right). %
\end{align*}
\citet[Theorem 1]{chatzigeorgiou2013bounds} states
\begin{align*}
W_{-1}(-e^{-u-1}) &\in -1 - \sqrt{2u} + \left[-u, -\frac{2}{3}u\right] \\
\implies -1 - W_{-1}(-e^{-u-1}) &\in \sqrt{2u} + \left[\frac{2}{3}u, u\right].
\end{align*}
Substituting yields
\begin{align}
\left(\tau + V_t\right) \left( -1 - W_{-1}\left(-z_t^{-1}\right) \right)
&\leq \left(\tau + V_t\right) \left( \sqrt{2 \log\left(\frac{z_t}{e^1}\right)} +  \log\left(\frac{z_t}{e^1}\right) \right). \label{eqn:iwstresone}
\end{align}
From \citet[Equation (9.8)]{fellerintroduction} we have
\begin{align*}
\Gamma(1 + n) &\in \sqrt{2 \pi n} \left(\frac{n}{e^1}\right)^n \left[e^{\frac{1}{12 n + 1}}, e^{\frac{1}{12 n}}\right] \\
\implies \left( \frac{\left(\tau + V_t\right)^{\tau + V_t}}{\Gamma\left(\tau + V_t\right)} \right)^{\frac{1}{\tau + V_t}} &\in \left(\frac{\tau + V_t}{2 \pi}\right)^{\frac{1}{2(\tau + V_t)}} e^1 \left[e^{-\frac{1}{12 (\tau + V_t)^2}}, e^{-\frac{1}{12 (\tau + V_t)^2 + (\tau + V_t)}}\right].
\end{align*}
Therefore
\begin{align}
\left(\tau + V_t\right) \sqrt{2 \log\left(\frac{z_t}{e^1}\right)}
&\leq \left(\tau + V_t\right) \sqrt{2 \log\left(\left(\frac{\tau + V_t}{2 \pi}\right)^{\frac{1}{2(\tau + V_t)}} e^{-\frac{1}{12 (\tau + V_t)^2 + (\tau + V_t)}} \phi_t(\tau, \alpha)^{\frac{1}{\tau + V_t}}\right)} \nonumber \\
&= \sqrt{2 \left(\tau + V_t\right) \log\left( \sqrt{\frac{\tau + V_t}{2 \pi}} e^{-\frac{1}{12 (\tau + V_t) + 1}} \phi_t(\tau, \alpha) \right)}, \label{eqn:iwpartone}
\end{align}
and
\begin{align}
\left(\tau + V_t\right) \log\left(\frac{z_t}{e^1}\right)
&\leq \left(\tau + V_t\right) \log\left(\left(\frac{\tau + V_t}{2 \pi}\right)^{\frac{1}{2(\tau + V_t)}} e^{-\frac{1}{12 (\tau + V_t)^2 + (\tau + V_t)}} \phi_t(\tau, \alpha)^{\frac{1}{\tau + V_t}}\right) \nonumber \\
&= \log\left(\sqrt{\frac{\tau + V_t}{2 \pi}} e^{-\frac{1}{12 (\tau + V_t) + 1}} \phi_t(\tau, \alpha)\right). \label{eqn:iwparttwo}
\end{align}
Combining \cref{eqn:iwstresone,eqn:iwpartone,eqn:iwparttwo} yields the crossing boundary
\begin{align*}
\frac{S_t}{t} &= \sqrt{2 \left(\frac{(\tau + V_t)/t}{t}\right) \log\left( \sqrt{\frac{\tau + V_t}{2 \pi}} e^{-\frac{1}{12 (\tau + V_t) + 1}} \left( \frac{1 + \alpha^{-1}}{C(\tau)} \right) \right) } \\
&\qquad + \frac{1}{t} \log\left(\sqrt{\frac{\tau + V_t}{2 \pi}} e^{-\frac{1}{12 (\tau + V_t) + 1}} \left( \frac{1 + \alpha^{-1}}{C(\tau)} \right)\right).
\end{align*}
\end{proof}
\end{lemma}

\section{Simulations}
\label{app:simulations}

\subsection{i.i.d. setting}
\label{app:iidsimulations}

For non-importance-weighted simulations, we use
the Beta-Binomial boundary of \citet{howard2021time} for $\lboracle_t$ and
$\uboracle_t$.
The curved boundary is induced by the
test NSM
$$
\begin{aligned}
W_t(b; \hat{q}_t, q_t) &= \frac{\int_{q_t}^1 d\text{Beta}\left(p; b q_t, b (1 - q_t)\right)\ \left(\frac{p}{q_t}\right)^{t \hat{q}_t} \left(\frac{1 - p}{1 - q_t}\right)^{t (1 - \hat{q}_t)}}{\int_{q_t}^1 d\text{Beta}\left(p; b q_t, b (1 - q_t)\right)} \\
&= \frac{1}{(1 - q_t)^{t (1 - \hat{q}_t)} q_t^{t \hat{q}_t}} \left(\frac{\text{Beta}(q_t, 1, b q_t + t \hat{q}_t, b (1 - q_t) + t (1 - \hat{q}_t))}{\text{Beta}(q_t, 1, b q_t, b (1 - q_t))}\right) \\
\end{aligned}
$$ with prior parameter $b=1$.  Further
documentation and details are in the
reference implementation \texttt{csnsquantile.ipynb}.

The importance-weighted simulations use
the constructions from \cref{app:importanceweighted}: the reference
implementation is in \texttt{csnsopquantile.ipynb}
for the DDRM variant and \texttt{csnsopquantile-ebern.ipynb} for the empirical Bernstein variant.

\begin{figure*}[p]
\centering
\begin{minipage}[t]{.49\textwidth}
  \vskip 0pt
  \centering
  \includegraphics[width=.96\linewidth]{betacurves.pdf}
  \vskip -12pt
  \repeatcaption{fig:betacurves}{\betacurvescaption}
\end{minipage}
\hfill
\begin{minipage}[t]{.49\textwidth}
  \vskip 0pt
  \centering
  \includegraphics[width=.96\linewidth]{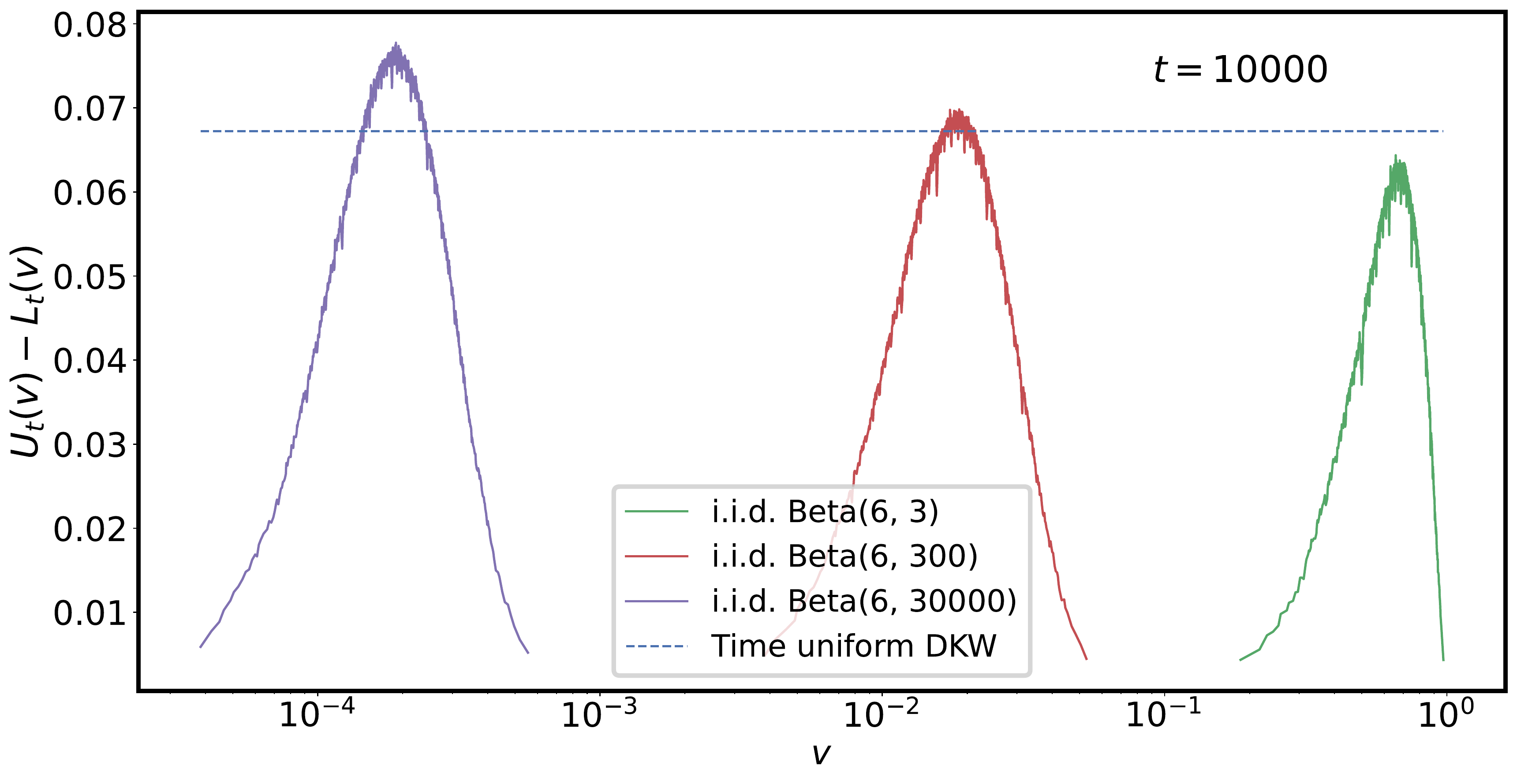}
  \vskip -12pt
  \caption{Comparison to naive time-uniform DKW (which is only valid in the i.i.d.\ setting) for Beta distributions of varying smoothness. Decreasing smoothness degrades our bound.}
  \label{fig:betawidth}
\end{minipage}
\vskip -0.2in
\end{figure*}

\begin{figure*}[p]
\centering
\begin{minipage}[t]{.49\textwidth}
  \vskip 0pt
  \centering
  \includegraphics[width=.96\linewidth]{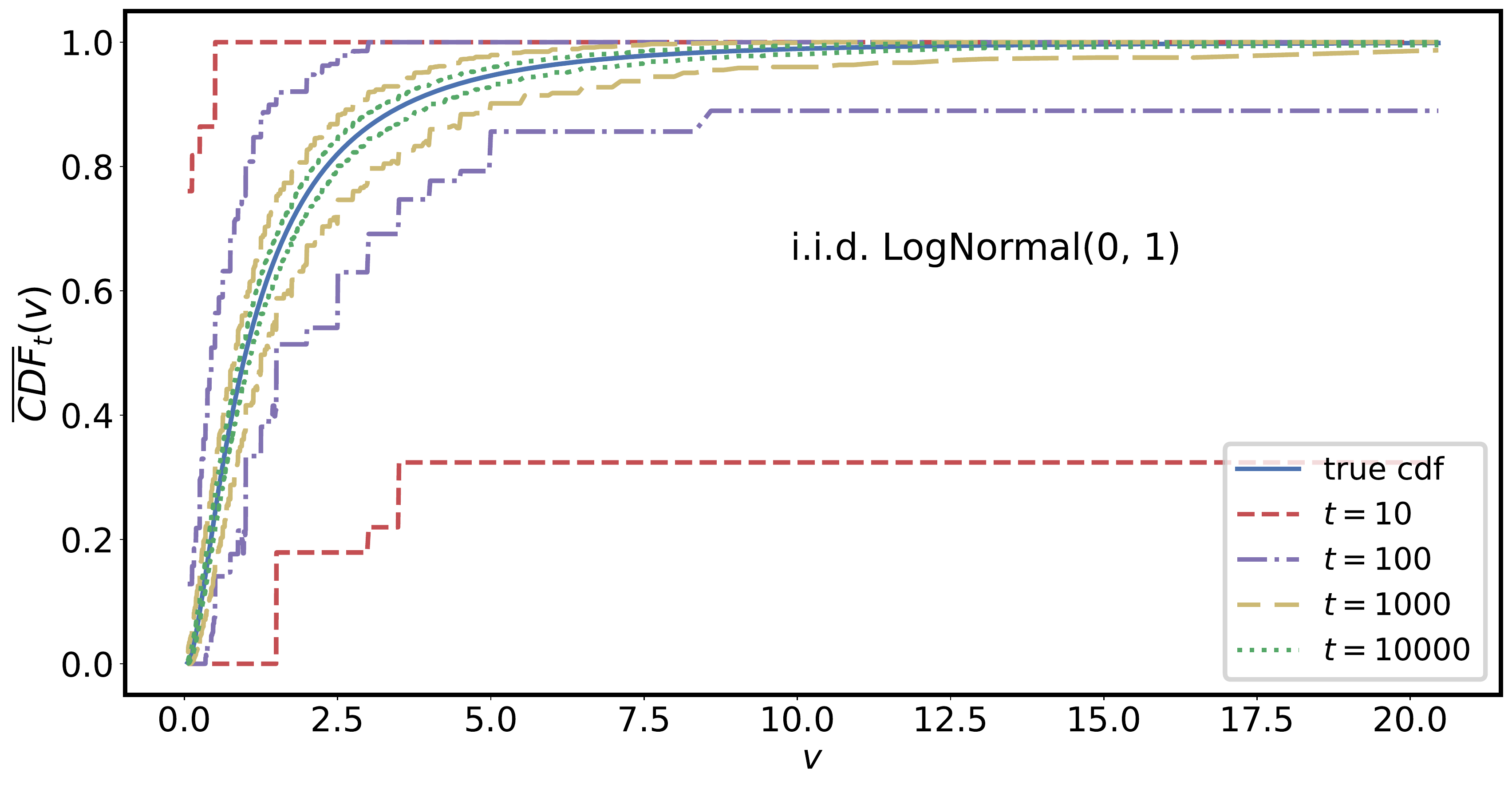}
  \vskip -12pt
  \caption{CDF bounds approaching the true CDF when sampling i.i.d.\ from a lognormal(0, 1) distribution. Recall these bounds
are simultaneously valid for all times and values.}
  \label{fig:lognormalcurves}
\end{minipage}
\hfill
\begin{minipage}[t]{.49\textwidth}
  \vskip 0pt
  \centering
  \includegraphics[width=.96\linewidth]{lognormalboundwidth.unbounded.pdf}
  \vskip -12pt
  \repeatcaption{fig:lognormalwidth}{\lognormalwidthcaption}
\end{minipage}
\vskip -0.2in
\end{figure*}

\begin{figure*}[p]
\centering
\begin{minipage}[t]{.49\textwidth}
  \vskip 0pt
  \centering
  \includegraphics[width=.96\linewidth]{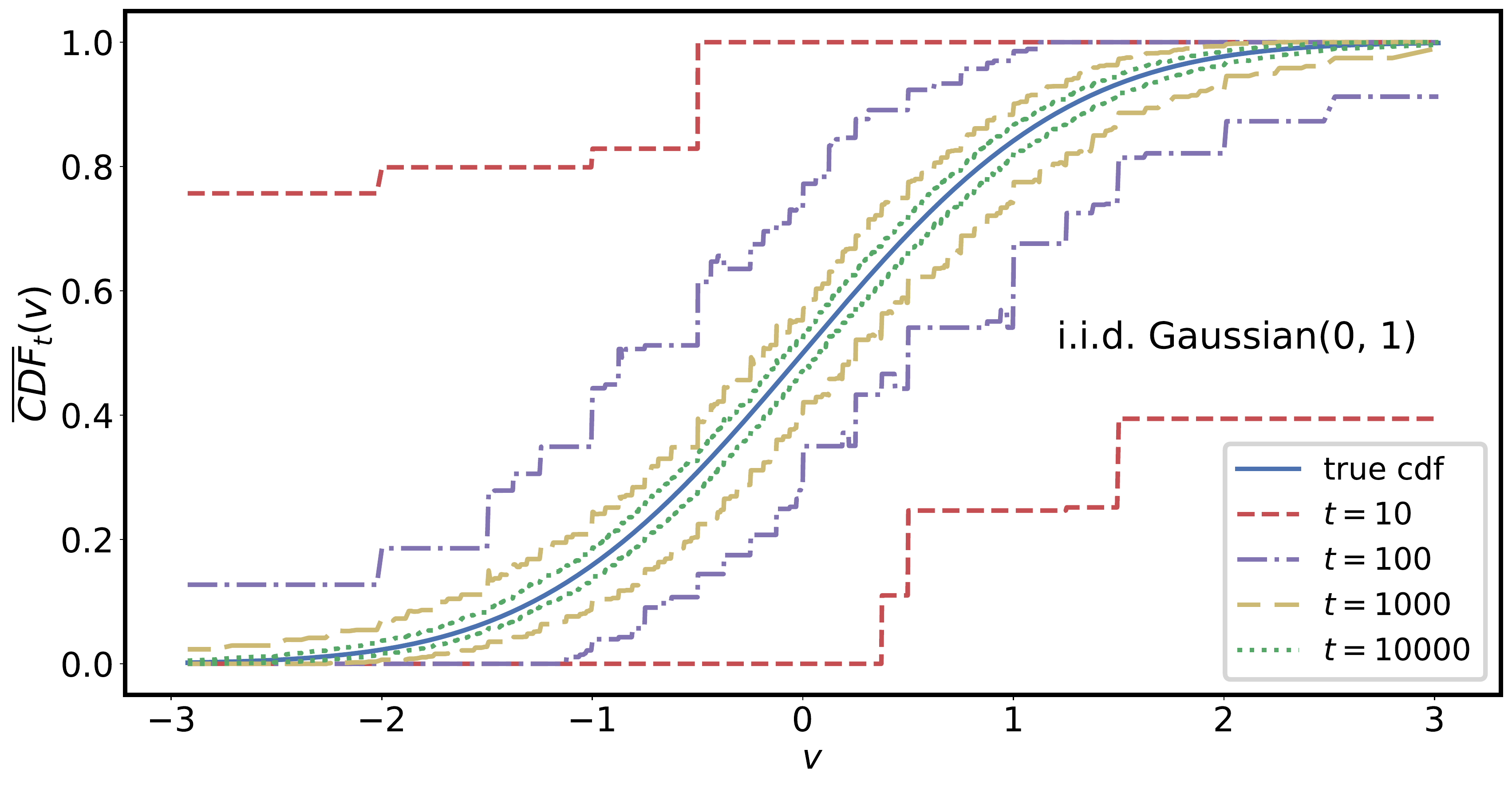}
  \vskip -12pt
  \caption{CDF bounds approaching the true CDF when sampling i.i.d.\ from a Gaussian(0, 1) distribution. Recall these bounds
are simultaneously valid for all times and values.}
  \label{fig:gaussiancurves}
\end{minipage}
\hfill
\begin{minipage}[t]{.49\textwidth}
  \vskip 0pt
  \centering
  \includegraphics[width=.96\linewidth]{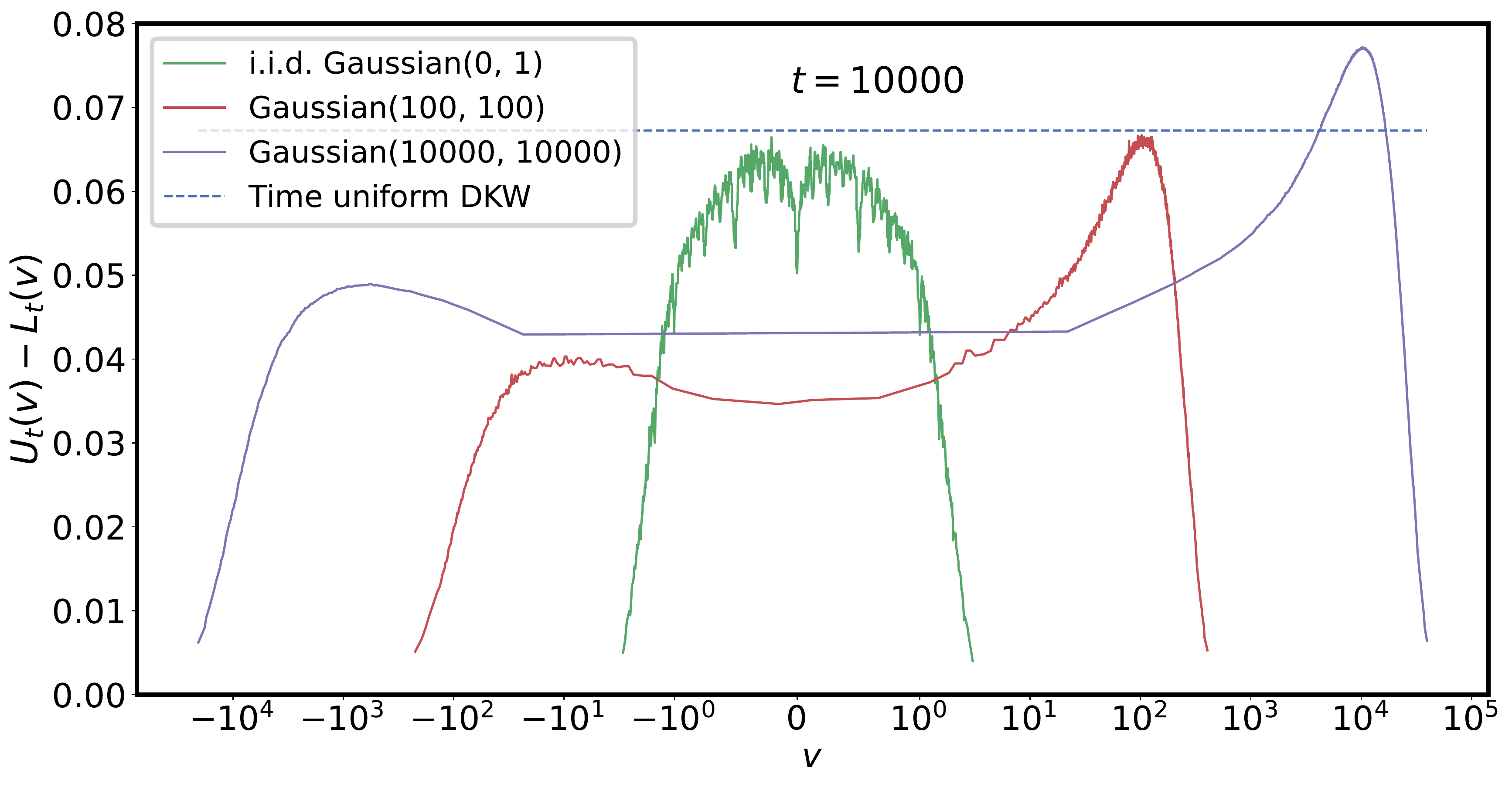}
  \vskip -12pt
  \caption{Demonstration of the variant described in \cref{subsec:extensions,app:arbitrarysupport} for distributions with arbitrary support, based on i.i.d.\ sampling from a variety of Gaussian distributions. Logarithmic range dependence is evident.}
  \label{fig:gaussianwidth}
\end{minipage}
\vskip -0.2in
\end{figure*}

\begin{figure*}[p]
\centering
\begin{minipage}[t]{.49\textwidth}
  \vskip 0pt
  \centering
  \includegraphics[width=.96\linewidth]{contpolyatwoseeds.pdf}
  \vskip -12pt
  \repeatcaption{fig:contpolyatwoseeds}{\polyacurvescaption}
\end{minipage}
\hfill
\begin{minipage}[t]{.49\textwidth}
  \vskip 0pt
  \centering
  \includegraphics[width=.96\linewidth]{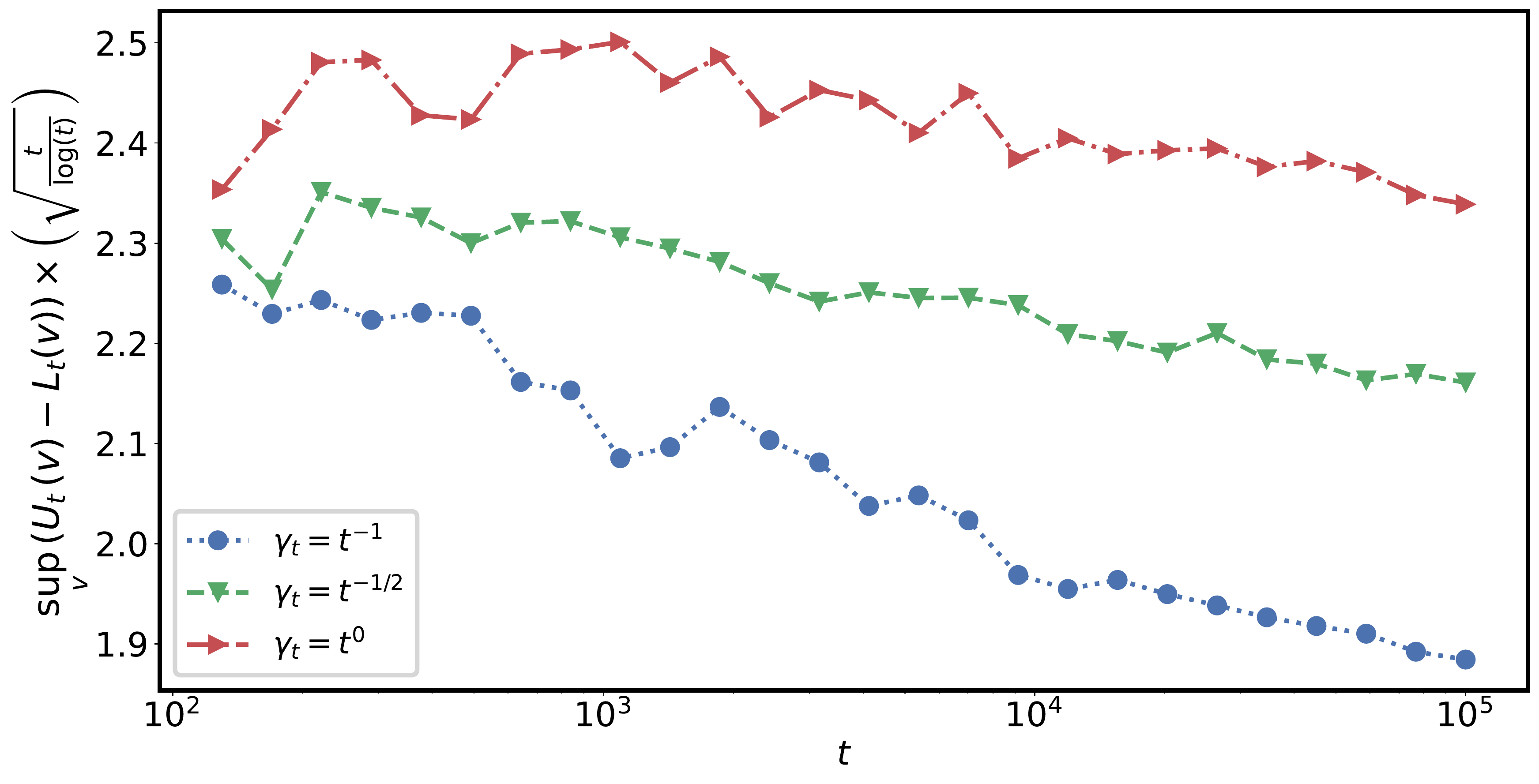}
  \vskip -12pt
  \caption{Maximum bound width, scaled by $\sqrt{\nicefrac{t}{\log(t)}}$
  to remove the primary trend, as a function of $t$, for nonstationary
  Polya simulations with different $\gamma_t$ schedules. See \cref{exp:nonstationary}}
  \label{fig:contpolyagammasweep}
\end{minipage}
\vskip -0.2in
\end{figure*}

\begin{figure*}[p]
\centering
\begin{minipage}[t]{.49\textwidth}
  \vskip 0pt
  \centering
  \includegraphics[width=.96\linewidth]{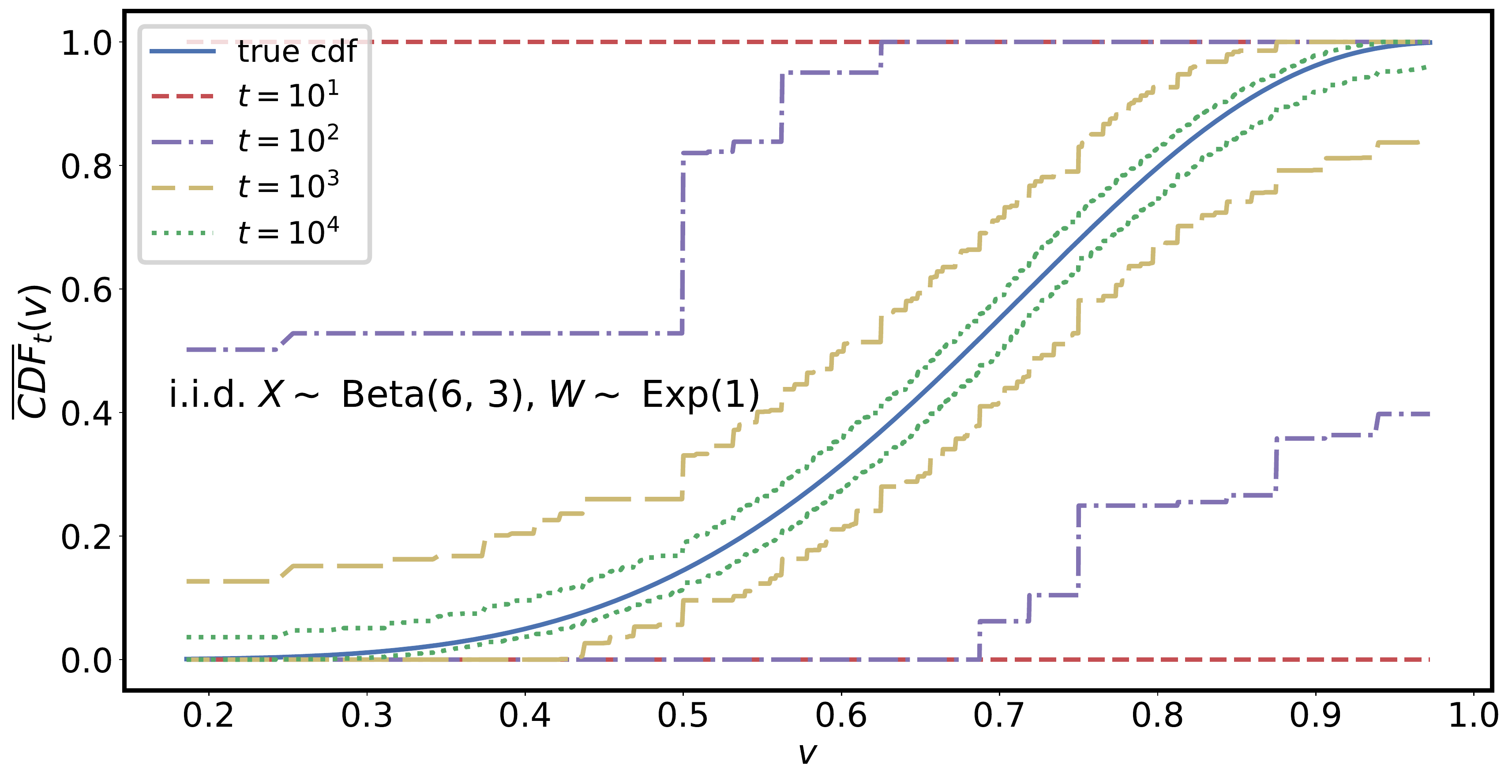}
  \vskip -12pt
  \caption{CDF bounds approaching the true counterfactual CDF when sampling i.i.d.\ from a Beta(6,3) with finite-variance importance weights, using DDRM for the oracle confidence sequence.}
  \label{fig:iwexpbetacurves}
\end{minipage}
\hfill
\begin{minipage}[t]{.49\textwidth}
  \vskip 0pt
  \centering
  \includegraphics[width=.96\linewidth]{betaparetoiwcurves.pdf}
  \vskip -12pt
  \repeatcaption{fig:iwexpparetocurves}{\iwbetaparetocurvescaption}
\end{minipage}
\vskip -0.2in
\end{figure*}

\begin{figure*}[p]
\centering
\begin{minipage}[t]{.49\textwidth}
  \vskip 0pt
  \centering
  \includegraphics[width=.96\linewidth]{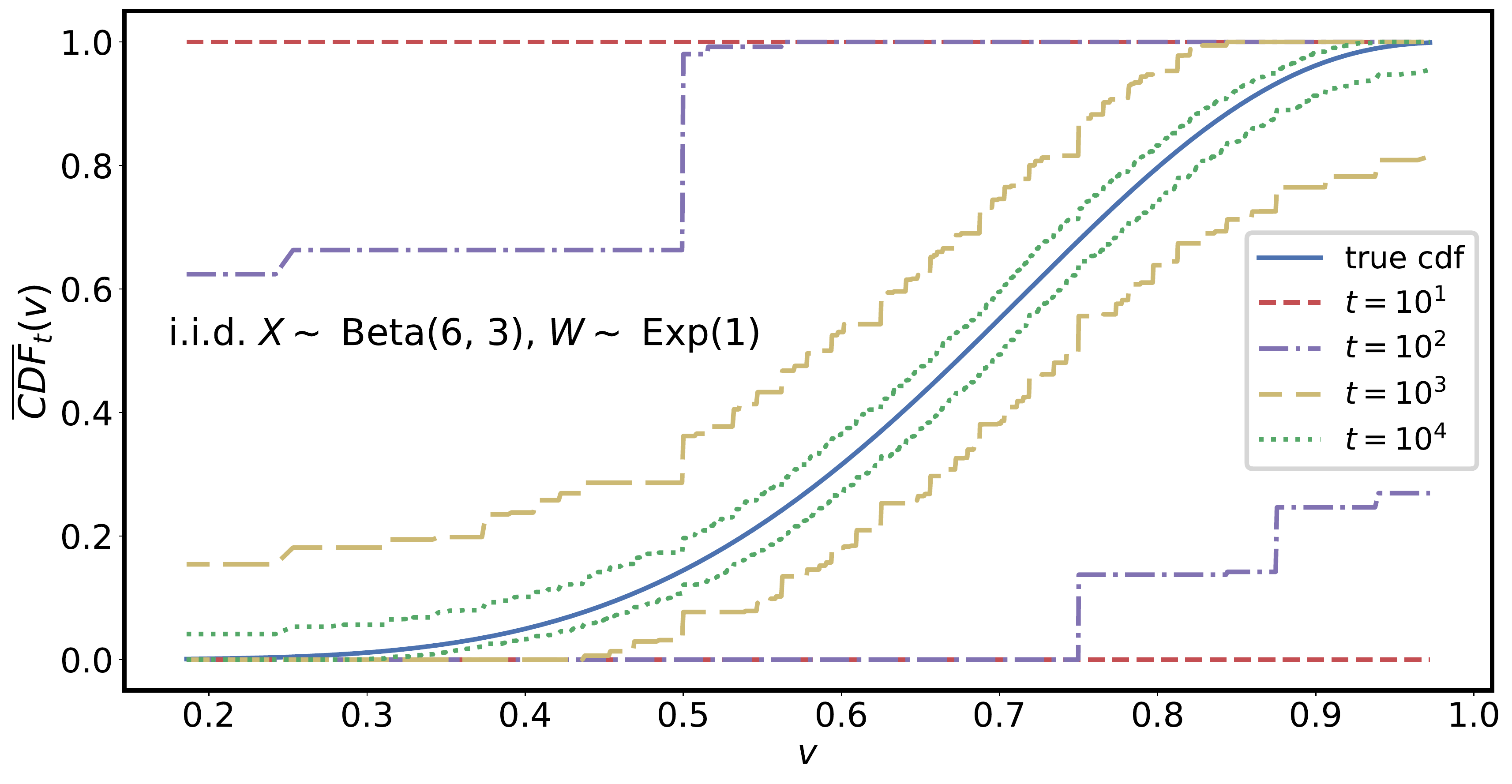}
  \vskip -12pt
  \caption{CDF bounds approaching the true counterfactual CDF when sampling i.i.d.\ from a Beta(6,3) with finite-variance importance weights, using Empirical Bernstein for the oracle confidence sequence.}
  \label{fig:iwexpbetacurvesbern}
\end{minipage}
\hfill
\begin{minipage}[t]{.49\textwidth}
  \vskip 0pt
  \centering
  \includegraphics[width=.96\linewidth]{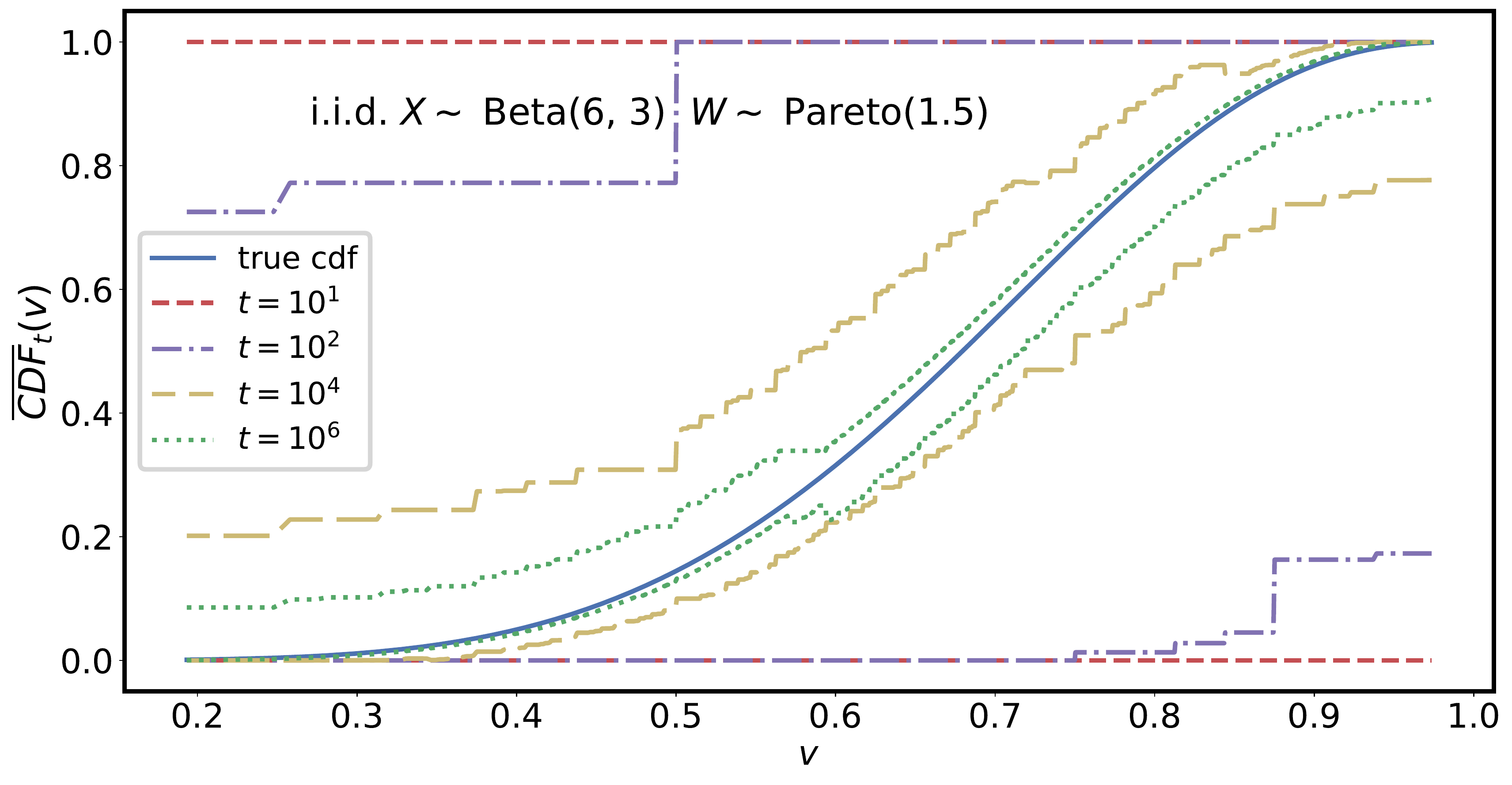}
  \vskip -12pt
  \caption{CDF bounds approaching the true counterfactual CDF when sampling i.i.d.\ from a Beta(6,3) with infinite-variance importance weights, using Empirical Bernstein for the oracle confidence sequence. Despite apparent convergence, eventually this simulation would reset the Empirical Bernstein oracle confidence sequence to trivial bounds.}
  \label{fig:iwexpparetocurvesbern}
\end{minipage}
\vskip -0.2in
\end{figure*}